%% file: article.tex
\documentclass[onefignum,onetabnum]{siamonline250211}

\usepackage{lipsum}
\usepackage{amsfonts}
\usepackage{graphicx}
\usepackage{epstopdf}
\usepackage{algorithmic}
\usepackage{tikz}
\usepackage[normalem]{ulem}
\ifpdf
  \DeclareGraphicsExtensions{.eps,.pdf,.png,.jpg}
\else
  \DeclareGraphicsExtensions{.eps}
\fi

\usepackage{enumitem}
\setlist[enumerate]{leftmargin=.5in}
\setlist[itemize]{leftmargin=.5in}

\newsiamremark{remark}{Remark}
\newsiamremark{hypothesis}{Hypothesis}
\crefname{hypothesis}{Hypothesis}{Hypotheses}
\newsiamthm{claim}{Claim}
\newsiamremark{fact}{Fact}
\crefname{fact}{Fact}{Facts}

\newcommand{\R}{\mathbb{R}}

\headers{Target localization using latent symmetries}{D. Dukov, M. Röntgen, and B. Davies}

\title{Target localization, identification and sensing using latent symmetries\thanks{Submitted to the editors \today.
\funding{M.R. acknowledges financial support from the ``Eastern Institute for Advanced Study Postdoctoral Excellent Programme''.}}}

\author{David Dukov\thanks{Mathematics Institute, University of Warwick, Coventry CV4~7AL, UK 
  (contact: \email{bryn.davies@warwick.ac.uk}).}
\and Malte Röntgen\thanks{Eastern Institute for Advanced Study, Eastern Institute of Technology Ningbo, Zhejiang, China.}
\and Bryn Davies\footnotemark[2]}

\usepackage{amsopn}

\ifpdf
\hypersetup{
  pdftitle={Target localization using latent symmetries}, 
  pdfauthor={D. Dukov, M. Röntgen, and B. Davies}
}
\fi

\usepackage{todonotes}

\begin{document}

\maketitle

\begin{abstract}
We show that an array of scatterers which has been designed to have latent (``hidden") symmetries can be used as a sensor. We use the capacitance matrix as a canonical model for three-dimensional hybridisation and study how the introduction of an ``intruder'' scatterer breaks the latent symmetries. By analysing the degree to which each symmetry is broken, we identify the radius of the intruder and localize its position. This can be achieved using a dictionary-based approach, however Bayesian inference or an artificial neural network (multi-layer perceptron) perform better in the presence of measurement noise. To our knowledge, this is the first time latent symmetries have been exploited successfully for sensing problems. It is also the first time latent symmetries have been observed in a three-dimensional open system that cannot be approximated by a sparse graph.
\end{abstract}

\begin{keywords}
inverse scattering, imaging, hidden symmetry, capacitance matrix, resonant modes
\end{keywords}

\begin{MSCcodes}
15A18, 78A46
\end{MSCcodes}

\section{Introduction}

Symmetry is a fundamental property of physical systems. Understanding the symmetries of a domain or of the governing equations is often one of the most effective ways to gain intuition about a scientific problem. The field of imaging and sensing is no exception, as symmetry is key to a wide range of methods including reconstruction algorithms, feature extraction, invariant representations, pattern recognition and other inverse problems \cite{SchonliebReview, datchev2011inverse}. 

Recently, there has been growing excitement about models that don't have any classical geometric symmetries but, nevertheless, display behaviours typical of symmetric systems. This is made precise through the notion of \emph{latent symmetry}, which is when a model has eigenfunctions that are ``symmetric'' in the sense of having parity at specific sites \cite{smith2019hidden}. For example, the eigenfunctions on the left in Figure~\ref{fig:introfig} are such that the values on sites 3 and 7 are all equal up to a change of sign. Latent symmetry is sometimes known as ``hidden'' symmetry since it doesn't require the system to have any classical geometric symmetries so can appear in unexpected places.

Originally proposed for the analysis of large graphs such as social networks \cite{bunimovich2019finding, smith2019hidden}, latent symmetries have since attracted significant interest from physicists. They have been investigated in a variety of setups, ranging from tight-binding systems \cite{Rontgen2021PRL126180601LatentSymmetryInducedDegeneracies,Kempkes2023QF21CompactLocalizedBoundaryStates,brandaoLatentSymmetryMinimal2026}, acoustics \cite{guoObservationChiralEdge2025,Rontgen2023PRL130077201HiddenSymmetriesAcousticWave}, and transmission line networks \cite{Rontgen2023PRA20044082EquireflectionalityCustomizedUnbalancedCoherent}, to topological systems \cite{Zheng2023PRB108L220303RobustTopologicalEdgeStates,Cui2023PRL131237201ExperimentalRealizationStableExceptional,linTopologicalAndersonInsulators2026,Rontgen2024PRB110035106TopologicalStatesProtectedHidden,eekHigherorderTopologyProtected2025}. Proposed applications include secure message transfer \cite{Sol2024AM362303891CovertScatteringControlMetamaterials} and the transfer for quantum states \cite{himmelStateTransferLatentsymmetric2026}.

In this work, we propose an exciting new application of latent symmetries to sensing problems by means of an open system of scatterers in three dimensions. We use the capacitance matrix as a canonical model to develop a proof-of-concept demonstration \cite{diaz2011positivity, ammari2024functional}. This was introduced to model the distributions of potential and charge in an array of conductors \cite{maxwell1873treatise} and has recently been generalised to describe the low-frequency (subwavelength) coupled resonances of a collection of acoustic resonators \cite{ammari2024functional}. An important subtlety of this model is that it is a dense matrix with relatively slowly decaying off-diagonal entries. This describes the fact that there is non-trivial long-range coupling between scatterers. This coupling is due to the relatively slow decay of the Green's function and is typical of systems in electrostatics or (low-frequency) acoustics. It means that  the capacitance matrix cannot be well approximated as a sparse graph, making it distinctly different from the systems in which latent symmetries have been observed previously.

The starting point for this study is the discovery of asymmetric arrangements of scatterers for which the capacitance matrix exhibits latent symmetries (Section~\ref{sec:latsym_examples}). This is groundbreaking in itself, as it is the first time latent symmetries without any classical geometric symmetries have been observed in a three-dimensional open system that cannot be approximated by a sparse graph (due to the non-trivial long-range coupling between scatterers). We will present several such examples here, one of which is shown in Figure~\ref{fig:introfig}.

\begin{figure}[tbh]
  \centering
  \label{fig:sketch}
  \includegraphics[width=\linewidth]{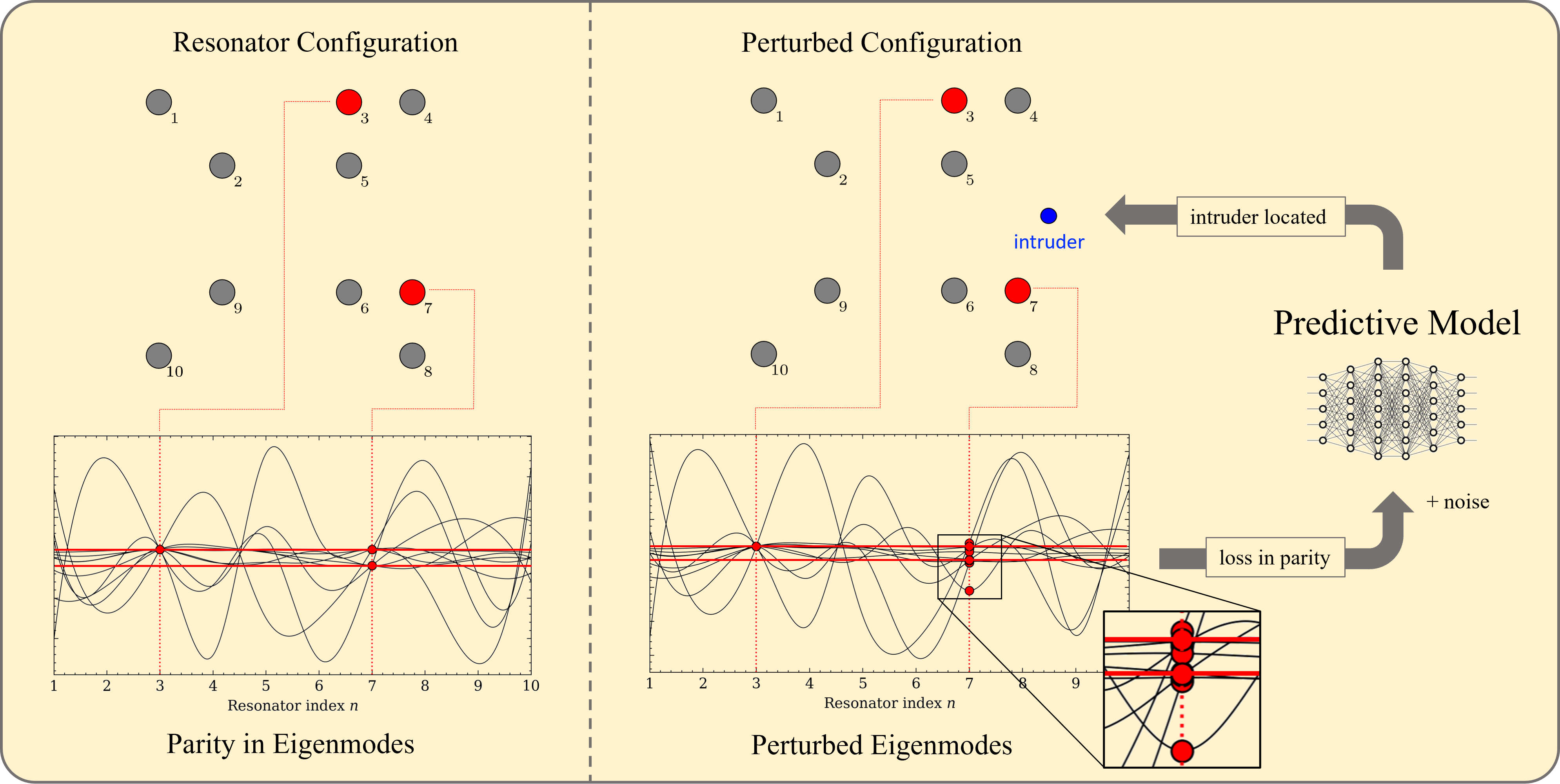}
  \caption{The resonator geometry shown on the left is such that the capacitance matrix supports a latent symmetry at sites (resonators) 3 and 7. The plot shows the discrete values of the eigenvectors at the 10 sites, interpolated using polynomial splines for illustrative purposes.  As a result of the latent symmetry, the eigenmodes display a parity symmetry at those positions, demonstrated through the two red lines that mark amplitude values $+1$ and $-1$.
  This latent symmetry is in spite of the system not having any geometric symmetries. In this study, we use such an array as a sensor by measuring how the parity of eigenmodes is broken by the presence of an intruder (right panel). We develop a predictive model that is capable of estimating the position and nature of the intruder from measurements of the broken parity, even in the presence of noise.}
  \label{fig:introfig}
\end{figure}

The idea explored in this work is that a latently symmetric system can be used as a sensor as it can detect any phenomena that break the symmetry of the system. For example, in Figure~\ref{fig:introfig} we show a resonator configuration for which the sites numbered 3 and 7 are latently symmetric. In the initial resonator configuration, on the left, we can see that the eigenmodes are all equal at these sites, up to parity (i.e. up to a change of sign). When an intruder is added, on the right in the figure, we see that the parity is broken; the modes shown are normalised so that the value on site 3 is the same as on the left-hand side of the figure, but the value on site 7 is no longer constrained to be $\pm$ this value.

There are several existing approaches to building sensors based on the principle of designing a resonator array with a specific property, then studying how these properties are perturbed by external phenomena. Many of these approaches use the shift of a particular eigenfrequency. Examples using whispering gallery modes have even been sufficiently sensitive that they can detect a single virus \cite{vollmer2008single}. Nevertheless, there are limits on how much the eigenvalues can change when a system is perturbed (see \cite{davies2022robustness} for estimates in the capacitance matrix setting). One solution to this has been exploiting non-Hermitian systems which can break these limits in the neighbourhood of exceptional points \cite{ammari2021high, cao2025exceptional, heiss2012physics, rechtsman2017optical}, however these are challenging to fabricate. In this work, we extract additional information from the resonator array by measuring richer information than just the shift of an eigenfrequency, by quantifying how much a given symmetry is broken. This will allow us not only to sense the presence of an intruder, but also to identify it and localize its position.

Latent symmetries provide an efficient way to extract symmetry information. Identifying a small number of latently symmetric pairs of sites gives a low-dimensional representation of the symmetries of the system, which we can then use to sense, identify and localize the intruder. We will show that this can be achieved straightforwardly using a dictionary-based approach, where the measured values (of how much each latent symmetry is broken) is compared to some known dictionary of intruder radii and positions. However, when significant noise is added to the data (to represent measurement errors) we see that Bayesian inference or an artificial neural network (multi-layer perceptron) perform better \cite{chen2020review}.

The paper is organized as follows. In \Cref{sec:latsyms} we will recall the relevant properties of latent symmetry and present examples of resonator configurations for which the capacitance matrix supports latent symmetries. In \Cref{ch: pertResSys} we will study the forward problem to understand how the introduction of an intruder breaks the latent symmetries of our arrays. Finally, in \Cref{sec:sensing} we will use this insight to develop sensing algorithms, that can identify and localize the position of the intruder.

\input{2_LatentSymmetry}

\input{3_Perturbation}

\input{4_Detection}

\input{5_Conclusions}

\section*{Acknowledgments} 
M.R. acknowledges discussions with M. Pyzh, C. V. Morfonios, P. Schmelcher, W. Gao, and F. Y. Zhang. 
M.R. acknowledges financial support in the framework of the ``Eastern Institute for Advanced Study Postdoctoral Excellent Programme''.

\section*{Author Contributions}
D.D.: Investigation, Development of software, Visualization of results, Writing of manuscript.
M.R.: Supervision.
B.D.: Conceptualization, Supervision, Writing of manuscript.

\bibliographystyle{siamplain}
\bibliography{references, Bibtex_Malte}


\end{document}

%% file: 2_LatentSymmetry.tex
\section{Problem formulation}
\label{sec:latsyms}

We will consider the capacitance matrix as a canonical model for a coupled resonator system. This was first introduced by Maxwell to relate the distributions of potential and charge in an array of conductors \cite{maxwell1873treatise}. The model was also used to characterise the low-frequency (subwavelength) coupled resonances of a collection of high-contrast acoustic resonators (such as air bubbles in water) \cite{ammari2024functional}. Suppose that the resonators occupy the disjoint domains $D_1, D_2, ...,D_N\subset\mathbb{R}^3$, and let $D=D_1\cup...\cup D_N$. Suppose further that $\partial D\in C^{1,s}$ with $s\in(0,1)$. The capacitance matrix can be defined in terms of boundary integral operators. The single layer potential $S_D:L^2(\partial D)\rightarrow H^1_\text{loc}(\R^3)$ is given by
\begin{equation}
    S_D[\phi](x) = -\frac{1}{4\pi}\int_{\partial D}\frac{1}{|x-y|}\phi(y) \, d\sigma(y),
    \label{eq: single layer formula}
\end{equation}
for $x\in\R^3$ and $\phi\in L^2(\partial D)$. The entries of the capacitance matrix are then defined as
\begin{equation}
    \mathcal C_{ij} = -\int_{\partial D_i}S_D^{-1}[\chi_{\partial D_j}] \, d\sigma
    \label{eq: capacitance entries}
\end{equation}
for $i,j\in\{1,2,...,N\}$, where $\chi_{\partial D_j}$ is the indicator function on the boundary of resonator $D_j$.

A convenient approximation is to consider a dilute regime for the system, where the resonators are small compared to the distances between them. We will also make the assumption that all the resonators are identical, as this simplifies the notation and will be sufficient for the present study. That is, we suppose that the resonators occupy the regions $D_i = \epsilon B + z_i$, where $B\subset\R^3$ is a fixed bounded simply connected domain and $z_i\in\R^3$. Here, we assume that $\epsilon$ is small, typically at least one order of magnitude smaller than the size of $B$. In this regime, the capacitance matrix \eqref{eq: capacitance entries} can be approximated by the dilute capacitance matrix \cite{ammari2024functional}. This dilute matrix is defined as
\begin{equation}
    C_{ij}= \left \{
    \begin{array}{ll}
        \epsilon \text{Cap}_{B} &  \text{if} \ i=j,\\[6pt]
        \displaystyle  -\epsilon^2\frac{(\text{Cap}_{B})^2}{4\pi|z_i-z_j|}  & \text{if} \ i\ne j,
    \end{array}
    \right.
    \label{eq: dilute capacitance matrix}
\end{equation}
where Cap$_{B}=-\int_{\partial B}S^{-1}_{B}[\chi_{\partial B}]\, d\sigma$, and it is called the capacitance of domain $B$. For example, the capacitance of a spherical resonator of radius $R$ is $4\pi R$. The dilute capacitance matrix is an approximation to the capacitance matrix in the sense that $\mathcal C_{ij}=C_{ij}+ \mathcal{O}(\epsilon^3)$ as $\epsilon\to0$ \cite[Lemma 4.3]{ammari2020topologicallyProtected}. Hence, it is a convenient way to approximate the eigenstates of $\mathcal C$, when the system is sufficiently dilute.

\subsection{Latent symmetries}

When a domain has a geometric symmetry, this is often reflected in the properties of associated physical systems. In our case, the eigenvectors of the capacitance matrix at the symmetric sites are the same up to parity (\emph{i.e.} they are the same up to factor of $\pm1$). The same is true for the solutions of the governing PDEs, when studying the full acoustic or electrostatic problems. This is often described more concisely by the commutation $[H,M]=0$, where $H$ is the Hamiltonian of the system and $M$ is the matrix describing the symmetry operation\footnote{The square brackets here denote the standard matrix commutator $[A,B]:=AB-BA$.}. However, this property can also hold at sites which are not geometrically symmetric. This is know as a \emph{latent symmetry}, a more general notion of symmetry \cite{smith2019hidden}.

Latent symmetries have been observed in a variety of different physical settings, including acoustic waveguides \cite{rontgen2023HiddenSym} and photonic networks \cite{himmelEigenmodesLatentSymmetricQuantum2025}. The principles of latent symmetry can be applied for any system that can be reduced to a discrete eigenvalue problem
\begin{equation} \label{eq:Hamilt}
    H{Y} = \lambda{Y}.
\end{equation}
If the system is asymmetric, then we have that $[H,M]\ne 0$ for any matrix $M$ describing a symmetry operation. However, it is possible that a subset of two (or more) sites may exhibit some symmetry; this is a latent symmetry. If we denote the set of sites that are candidates for being latently symmetric as $S$, then we can split all the sites into $S$ and its complement $\overline{S}$ and write the eigenvalue problem \eqref{eq:Hamilt} in the block matrix form
\begin{equation}
    \begin{pmatrix}
        H_{SS} & H_{S\overline{S}} \\
        H_{\overline{S}S} & H_{\overline{SS}}\\
    \end{pmatrix}
    \begin{pmatrix}
        Y_S \\
        Y_{\overline{S}}\\
    \end{pmatrix}
    = \lambda
\begin{pmatrix}
        Y_S \\
        Y_{\overline{S}}\\
    \end{pmatrix}.
    \label{eq: HamiltonianSystem}
\end{equation}
We would then like to formulate an eigenvalue problem posed only on the sites in $S$.
Eliminating $Y_{\overline{S}}$ from the above system of two coupled equations leads us to the \emph{isospectral reduction} over $S$, given by
\begin{equation}
    R_S[H](\lambda) = H_{SS} - H_{S\overline{S}}(H_{\overline{SS}}-\lambda I)^{-1}H_{\overline{S}S},
    \label{eq: effectiveHamiltonian}
\end{equation}
where we have implicitly assumed that $H_{\overline{SS}}-\lambda I$ is invertible. Then the eigenvalue problem reduces to $R_S[H](\lambda)Y_S = \lambda Y_S$ and a latent symmetry is present at the sites of $S$ if there exists a permutation matrix $M$ such that $[R_S[H](\lambda),M]=0$ for all $\lambda$. In our case, the capacitance matrix $C$ plays the role of the Hamiltonian, and to check whether a pair of resonators $D_i$ and $D_j$ are latently symmetric we follow the above procedure. That is, letting $S=\{i,j\}$, we check that
\begin{equation}
    [R_S[C](\lambda),M]=0 \quad \text{with} \ M = \begin{pmatrix}
        0 & 1 \\
        1 & 0
    \end{pmatrix}.
\end{equation}
However, this process is very computationally inefficient, and a better approach is to consider a method stemming from graph theory. 

We can think of our eigenvalue problem as a weighted graph where the resonators are the vertices, and the matrix entries define the weights of the edges. Then, the capacitance matrix is the associated adjacency matrix and latently symmetric sites correspond to cospectral vertices \cite{godsil2024stronglyCospectral}. Hence, an equivalent condition for checking latent symmetry is given through the matrix power identity \cite{Kempton2020LAIA594226CharacterizingCospectralVerticesIsospectral,Rontgen2021AQSymmetriesMatrixItsIsospectral} 
\begin{equation} \label{eq:exponentiation}
    \left( C^k \right)_{ii} = \left( C^k \right)_{jj} \quad \text{for all} \ k=1,...,N-1,
\end{equation}
where the superscript $k$ denotes matrix exponentiation. This characterisation of latent symmetry underlies the numerical implementations presented later, with the dilute capacitance matrix $C$ used as the model.

\subsection{Latently symmetric arrays} \label{sec:latsym_examples}
The starting point for this study is finding resonator arrangements for which the dilute capacitance matrix is latently symmetric. We approached this with a straightforward search algorithm. For simplicity, we consider spherical resonators that are all placed on a single plane ($z=0$). Our algorithm randomly placed up to ten resonators on points on the grid $(x,y)\in\{1,\dots,5\}\times \{1,\dots,5\}$. For each configuration, it used the exponentiation formula \eqref{eq:exponentiation} to see if the capacitance matrix had a latent symmetry.

Some of the configurations for which the capacitance matrix was found to have at least one latent symmetry are shown in Figure~\ref{fig:lat-sym-examples}. Unsurprisingly, many of the configurations were geometrically symmetric, such as those in Figures~\ref{fig:lat-sym-examples}(c) and~\ref{fig:lat-sym-examples}(d). Nevertheless, the fact these support \emph{multiple} pairs of latent symmetries will prove to be useful to us later. Excitingly, the search algorithm also uncovered configurations that are latently symmetric but do not have any classical geometric symmetries, such as those in Figures~\ref{fig:lat-sym-examples}(a) and~\ref{fig:lat-sym-examples}(b). 


\begin{figure}
  \centering
  \label{fig:symmetry examples}
  \includegraphics[width=0.4\linewidth]{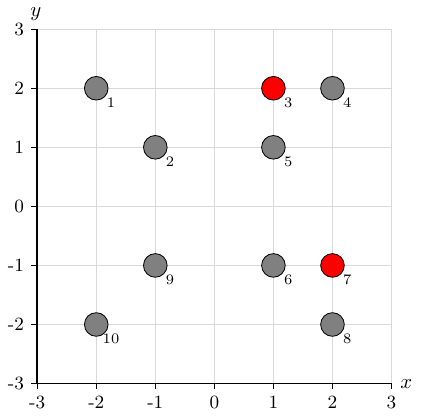}
  \includegraphics[width=0.4\linewidth]{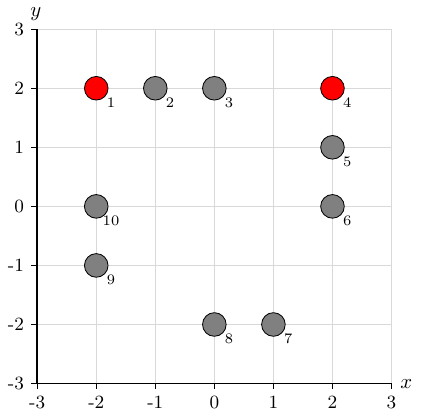}
  \makebox[0.4\linewidth][c]{\textbf{(a)}}
  \makebox[0.4\textwidth][c]{\textbf{(b)}} \\
  \includegraphics[width=0.4\linewidth]{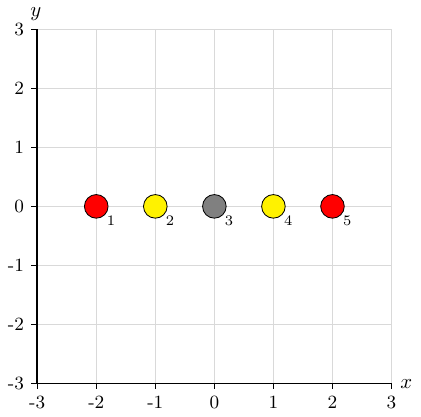}
  \includegraphics[width=0.4\linewidth]{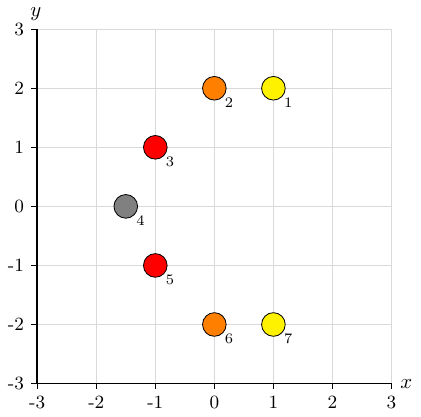}
  \makebox[0.4\linewidth][c]{\textbf{(c)}}
  \makebox[0.4\textwidth][c]{\textbf{(d)}} \\
  \caption{Examples of resonator arrays with latent and classical geometric symmetries. Plots (a) and (b) show arrays that are non-symmetric but have a pair of sites that are latently symmetric (these are sites 3 \& 7 in (a) and sites 1 \& 4 in (b)). Plots (c) and (d) show geometrically symmetric arrays which (trivially) have multiple pairs of latently symmetric sites.}
  \label{fig:lat-sym-examples}
\end{figure}

%% file: 3_Perturbation.tex
\section{Perturbations due to intruders} \label{ch: pertResSys}


In this section, we explore how the eigenvectors of the capacitance matrix are perturbed by the addition of a new ``intruder'' resonator. Suppose a new resonator appears in the system, occupying the region $\Omega = \epsilon \widetilde{B} + \tilde{z}$ for a bounded simply connected domain $\widetilde{B}\subset\R^3$ and $\tilde{z}\in\R^3$. We denote the perturbed system $\widetilde{D} = D\cup\Omega$, and we label the new resonator with index $N+1$. The new system can be described by an $(N+1)$-dimensional capacitance matrix. The new resonator means that a new column and row are added at the end of the matrix $C$ and the perturbed capacitance matrix has the block form
\begin{equation}
    \widetilde{C} = 
    \begin{bmatrix}
        C & A \\
        A^T & \tilde\alpha
    \end{bmatrix}, \quad \text{with} \ A_i=-\epsilon^2\frac{\text{Cap}_B\text{Cap}_{\widetilde{B}}}{4\pi|z_i - \tilde{z}|}, \ \tilde\alpha=\epsilon\text{Cap}_{\widetilde{B}}.
    \label{eq: perturbed dilute cap mat}
\end{equation}

We expect perturbations to break the symmetry in the system and for this to extend to the parity in the eigenvectors. Specifically, the absolute difference $||\tilde{v}_{kp}|-|\tilde{v}_{kq}||$ allows us to quantify the perturbation, where $\tilde{v}_{kj}$ is the $j$th element of the $k$th eigenvector. Since each eigenvector is perturbed by a varying degree, depending on the position of the intruder (and will be largely unaffected when it is close to a nodal point), we elect to consider an average over all of their contributions. In particular, we define the function
\begin{equation}
    F_{pq} = \frac{1}{N+1}\sum_{k=1}^{N+1}\Big|\left|\tilde{v}_{kp}\right| - \left|\tilde{v}_{kq}\right| \Big|,
    \label{def: F}
\end{equation}
which depends on the position and size of the intruder. To emphasise this, we will often write $F_{pq}=F_{pq}(r,x,y)$, where $r$ is the radius of the intruder and $(x,y)$ is its position on the plane $(x,y,0)\in\R^3$; since the physics governing the system is invariant to translations and rotations, we confine the arrangements to the $xy$\nobreakdash-plane near the origin. To avoid redundancy in notation, we will generally omit the last position coordinate for the centres of the spheres and write $z=(x,y)\in\R^2$ instead of $z=(x,y,0)\in\R^3$. A function $F_{pq}$ can be defined for each pair of symmetric resonators giving us a family of functions for a system with multiple latently symmetric pairs. We denote the family by 
\begin{equation}
    \boldsymbol{F} = \{F_{pq}:D_p \text{ and } D_q \text{ are symmetric}\}.
\end{equation}

\begin{remark}
    Note that $F_{pq} = F_{qp}$, so we do not gain any additional information by swapping the order of a symmetric pair of resonators.
\end{remark}

\begin{remark}
    The functions $F_{pq}$ are not just abstract quantities, but could be estimated relatively easily in a physical system. When a system has a latent symmetry at sites $p$ and $q$, if it is excited with the same amplitude at $p$ and $q$, then we expect it will continue to be symmetric at $p$ and $q$ as it evolves. Conversely, if the latent symmetry is broken, then $F_{pq}$ can be estimated by measuring the difference in amplitude between $p$ and $q$ as it evolves.
\end{remark}

The goal of the forward problem is to study and understand the surfaces these functions represent and to infer properties about their analytic forms. In particular, we will investigate how changes in size and position of the intruder $\Omega$ affect them. For simplicity, we shall restrict our analysis to collinear and planar arrangements of resonators. Furthermore, we will assume that all the resonators are spherical, with radius $R$ of those in the sensing array and the intruder having radius $r$.

Under these assumptions, our system can be written as $\widetilde{D}=D_1\cup...\cup D_N\cup\Omega$ with $D_i = \epsilon B + z_i$, for $z_i\in\R^2$ and $\Omega = \epsilon \widetilde{B} + \tilde{z}$, for some $\tilde z\in\R^2$. Letting $\alpha = \epsilon\text{Cap}_B$ and $\tilde{\alpha} = \epsilon\text{Cap}_{\widetilde{B}}$, the perturbed dilute capacitance matrix \eqref{eq: perturbed dilute cap mat} has the form
\begin{equation*}
    \widetilde{C} =
    \begin{bmatrix}
        \alpha & -\frac{\alpha^2}{4\pi|z_1-z_2|} & \dots & -\frac{\alpha^2}{4\pi|z_1-z_N|} & -\frac{\alpha\tilde{\alpha}}{4\pi|z_1-\tilde{z}|} \\
        -\frac{\alpha^2}{4\pi|z_2-z_1|} & \alpha & \dots & \vdots &  -\frac{\alpha\tilde{\alpha}}{4\pi|z_2-\tilde{z}|}\\
        \vdots & \vdots & \ddots & \vdots& \vdots \\
        -\frac{\alpha^2}{4\pi|z_N-z_1|} & \dots & \dots  & \alpha & -\frac{\alpha\tilde{\alpha}}{4\pi|z_N-\tilde{z}|} \\
        -\frac{\alpha\tilde{\alpha}}{4\pi|\tilde{z}-z_1|} & \dots & \dots & -\frac{\alpha\tilde{\alpha}}{4\pi|\tilde{z}-z_N|} & \tilde{\alpha}
    \end{bmatrix}.
\end{equation*}
The dependence of $\widetilde{C}$ on the intruder's radius and position is governed by $\tilde{\alpha}(r) = \epsilon4\pi r$ and $\tilde{z}=(x,y)$. This allows us to make the following propositions about the asymptotic behavior of the perturbed system in the limit of small or far away intruders.

\begin{proposition}
    Suppose $\Omega$ has a fixed position $\tilde{z}=(x_0,y_0)\in\R^2$. Then letting $r\rightarrow0$, the eigenvectors of $\widetilde{C}$ converge to $\boldsymbol{v}_k\oplus0$ and $\boldsymbol{0}_N\oplus 1$ , where $\boldsymbol{v}_k$ is an eigenvector of the unperturbed matrix $C$. Moreover, we have that $F_{pq}(r,x_0,y_0)\rightarrow0$ as $r\rightarrow0$.
    \label{prop: r tends to 0}
\end{proposition}

\begin{proof}
    The key here is that each entry in the last row or column of the perturbed matrix contains a factor of $\tilde{\alpha}(r) = \epsilon4\pi r$, which converges to zero as $r\rightarrow0$. Keeping in mind that the denominators in the coefficients of the last row and column are constant, since $\tilde{z}=(x_0,y_0)$ is fixed, we have that
    \begin{equation*}
        \left\lVert \widetilde{C}-C\oplus0_{1\times 1} \right\rVert^2_F = 2\sum_{i=1}^N \left(\frac{\alpha\tilde{\alpha}}{4\pi|z_i-\tilde z|}\right)^2 + \tilde{\alpha}^2.
    \end{equation*}
    Thus, 
    \begin{equation*}
        \lim_{r\rightarrow0}\widetilde{C}(r) =
        \begin{bmatrix}
            C  & \boldsymbol{0}_{N\times1} \\
            \boldsymbol{0}_{1\times N} & 0
        \end{bmatrix}.
    \end{equation*}
    The eigenvalues of this block matrix are $\lambda_1,...,\lambda_N,$ and $0$, where the $\lambda_k$ are the eigenvalues of the unperturbed matrix $C $. Thus, if $\boldsymbol{v}_k$ is an eigenvector of $C$ associated to $\lambda_k$, then the block vectors $\boldsymbol{v}_k\oplus0$ are eigenvectors of $\widetilde{C}$. For the eigenvector associated to the eigenvalue $0$, we need to consider the fact that $C$ is a symmetric positive definite matrix. Specifically, it has a full rank, and so the null space of $\widetilde{C}$ is one-dimensional with basis $\boldsymbol{0}_N\oplus1$. Lastly, we note that $F_{pq}\equiv0$ for the block matrix above since the eigenvectors $\boldsymbol{v}_k$ retain their parity and all the coefficients in $\boldsymbol{0}_N\oplus1$ associated to the symmetric resonators are zero.
\end{proof}

The significance of the above result is that the system will struggle to detect very small intruders. We also cannot have $r$ be too large as then the diluteness assumption in the physical model may be violated. Hence, it is desirable to consider detection systems where the intruder's radius $r$ is assumed to be within one order of magnitude from the resonators' radii $R$.

\begin{proposition}
    Suppose $\Omega$ has a fixed radius $r_0>0$. Then letting $|\tilde{z}|\rightarrow\infty$, the eigenvectors of $\widetilde{C}$ converge to $\boldsymbol{v}_k\oplus0$ and $\boldsymbol{0}_N\oplus1$, where $\boldsymbol{v}_k$ is an eigenvector of $C$. Moreover, $F_{pq}(r_0,x,y)\rightarrow0$ as $|\tilde{z}=(x,y)|\rightarrow\infty$.
    \label{prop: |(x,y)| tends to inf}
\end{proposition}
\begin{proof}
    The proof is similar to the argument for Proposition \ref{prop: r tends to 0}. The slight difference is that now $\tilde{\alpha}=\epsilon4\pi r_0$ is fixed, and the denominators in the coefficients of the last row and column of $\widetilde{C}$ vary. Considering the difference between the perturbed matrix and $C\oplus\tilde{\alpha}$ in the Frobenius norm, 
     \begin{equation*}
        \left\lVert \widetilde{C}-C\oplus\tilde{\alpha} \right\rVert^2_F = 2\sum_{i=1}^N \left(\frac{\alpha\tilde{\alpha}}{4\pi|z_i-\tilde z|}\right)^2,
    \end{equation*}
    we get that
    \begin{equation*}
        \lim_{|\tilde{z}|\rightarrow\infty} \widetilde{C} (\tilde{z}) =
        \begin{bmatrix}
            C  & \boldsymbol{0}_{N\times1} \\
            \boldsymbol{0}_{1\times N} & \tilde{\alpha}
        \end{bmatrix}.
    \end{equation*}
    Using the same reasoning as before, we find that the eigenvectors of this block matrix are $\boldsymbol{v}_k\oplus0$ and $\boldsymbol{0}_N\oplus1$. The only difference is that the last eigenvalue is $\tilde{\alpha}$ instead of $0$. Similarly, $F_{pq}\equiv0$ because of parity in the symmetric entries of all the eigenvectors.
\end{proof}

\begin{remark}
    We showed convergence along the radial and positional dimensions individually, but convergence also holds when they are both varied simultaneously. It is easy to see that $\lim_{r\rightarrow0}\lim_{|\tilde{z}|\rightarrow\infty}\widetilde{C} (r,\tilde{z})$ and  $\lim_{|\tilde{z}|\rightarrow\infty}\lim_{r\rightarrow0}\widetilde{C} (r,\tilde{z})$ are both equal to zero, and this can be shown along any path along which $|\tilde{z}|$ grows large while $r$ shrinks to zero.. Additionally, the above results hold for three-dimensional arrangements, where we have $\tilde{z}\in\R^3$ not fixed to a planar subspace.
\end{remark}

The consequence of Proposition~\ref{prop: |(x,y)| tends to inf} is that we will need to consider intruders $\Omega$ which are not too far from or too close to the system of resonators. Similarly to the case for small radius, it will be difficult to detect intruders far away from the symmetric system. Additionally, being too close to one of the resonators violates the diluteness assumption in the model.

This approach to sensing is different to the conventional far-field sensing in inverse scattering problems. The requirement for the intruder to be close to the resonator array has advantages for applications in very small scale environments but also presents challenges due to the complex scattering and interference taking place near to the resonators. We will observe this in the numerical simulations that follow.

\subsection{Varying radius}\label{sec: varying radius}

Let us consider the example from Figure~\ref{fig:lat-sym-examples}(c) again, with five collinear resonators, but this time we perturb the system by adding a new resonator at position $\tilde{z}=(3,0)$. These correspond to $D=D_1\cup...\cup D_5$ and $\Omega$, respectively, as depicted in Figure~\ref{fig: 5 chain preturbed}. The resonators all have radius 0.1 and their centres are separated by a distance 1. The idea is to vary the radius of $\Omega$, $r_\Omega=\epsilon r$ (by varying $r$) and observe its effect on the functions $F_{15}$ and $F_{24}$ associated to the symmetric pairs $\{D_1,D_5\}$ and $\{D_2,D_4\}$, respectively. We vary $r$ according to a uniform discretisation of the interval $[0,1]$, avoiding going any larger so as to be consistent with the diluteness assumption.

\begin{figure}
\centering
\begin{tikzpicture}[scale=2] 
  \foreach \i in {1,2,3,4,5} {
    \filldraw[fill=gray!40, draw=black] (\i,0) circle (0.15); 
    \node[above] at (\i,0.2) {\small $D_\i$};
  }
  \filldraw[fill=black, draw=black] (6,0) circle (0.075);
  \node[above] at (6,0.2) {\small $\Omega$};
\end{tikzpicture}
\caption{Five spherical resonators of radius $0.1$ positioned at the coordinates $z_i \in \{-2,-1,0,1,2\}\times\{0\}$ and an intruder $\Omega$ of variable radius at $(3,0)$. The symmetric pairs of the unperturbed system are $\{D_1, D_5\}$ and $\{D_2, D_4\}$.}
\label{fig: 5 chain preturbed}
\end{figure}
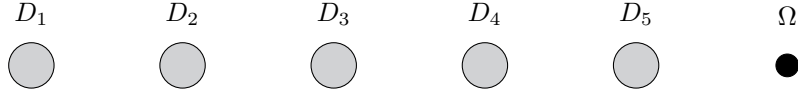

The numerical simulation presented in Figure~\ref{fig: radius1} shows that as $r$ approaches zero so do $F_{15}(r,3,0)$ and $F_{24}(r,3,0)$. This is consistent with the mathematical analysis of Proposition~\ref{prop: r tends to 0}. In the opposite direction where $r$ approaches $1$, that is $r_\Omega$ approaches $0.1$, we see that both functions grow smoothly up to around $r_\Omega=0.08$ for $F_{15}$ and $r_\Omega=0.09$ for $F_{24}$, after which, slightly more volatile trajectories are observed, characterised by a dip followed by another increase. Perhaps, one reason for this could be that the waves scattered off the intruder interfere with the symmetric resonators closer to it, namely $D_5$ and $D_4$. Since $D_5$ is closer, it is coupled with the intruder more strongly and experiences the interference effect first. Looking at the decay of the off-diagonal capacitance coefficients \eqref{eq: dilute capacitance matrix}, we can see that the coupling strength scales in proportion to the reciprocal of the distance between the two objects. The log-log plot on the right in Figure~\ref{fig: radius1} demonstrates linear behavior for small radius. Between $r_\Omega=10^{-3}$ and $r_\Omega=10^{-2}$, the slope is effectively a straight line which is not too surprising given that the functions $F_{pq}$ are governed by eigenvector perturbations which are in turn asymptotically linear for small perturbations.

\begin{figure}
\centering
\includegraphics[width=0.8\textwidth]{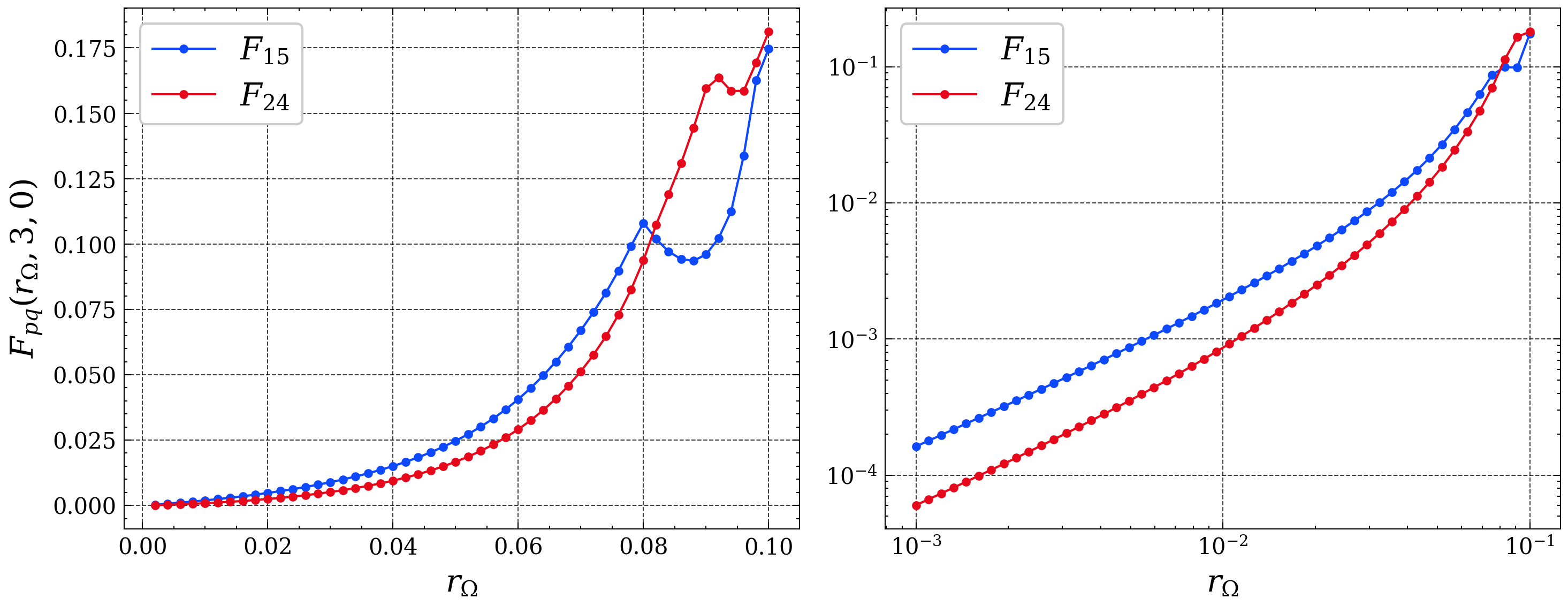}
\caption{Perturbations of system $D$ by an intruder $\Omega$ of variable radius fixed at position $(3,0)$. The size of $\Omega$ of the intruder does not exceed that of the five resonators of the original system $D$. The plot on the right is in log-log scale.}
\label{fig: radius1}
\end{figure}

We have observed the behaviour of $F_{15}$ and $F_{24}$ at the fixed position $(3,0)$. The question now is whether this behaviour can also be observed at other positions. We will demonstrate later that the functions in the family $\boldsymbol{F}$ are all symmetric in the position variable $\tilde{z}=(x,y)\in\R^2$ with respect to the mirror symmetries of the system, but for now, let us just assume so. The five resonators have two lines of symmetry - the $x$\nobreakdash-axis and $y$\nobreakdash-axis. Hence, it is sufficient to restrict our attention to just the upper right quadrant. To answer the question, we consider a sample of $200$ positions chosen uniformly random over the domain $[0,5]\times[0,3]$, and a curve in $r$ is plotted for each one of them. The simulated curves are shown in Figure~\ref{fig: radpos1}, colour-coded according to their respective positions with blue marking samples nearer the system of resonators and red marking samples further away.

Two important observations can be made from this simulation. Firstly, we see that the blue curves appear predominantly above the red curves, indicating that intruders further away from the resonators have lesser effect on the functions in $\boldsymbol{F}$. This is consistent with Proposition~\ref{prop: |(x,y)| tends to inf}. Secondly, and more interestingly, the curves appear to be parallel to each other. As the plot is in log-log scale, this implies that they are scalar multiples of each other with the scalar depending on the position of the intruder. This suggests the functions quantifying the effect of the perturbation are separable into $F_{pq}(r,x,y) = H_{pq}(r)G_{pq}(x,y)$.

\begin{figure}
\centering
\includegraphics[width=0.66\textwidth, trim=0cm 10.2cm 0cm 0cm, clip]{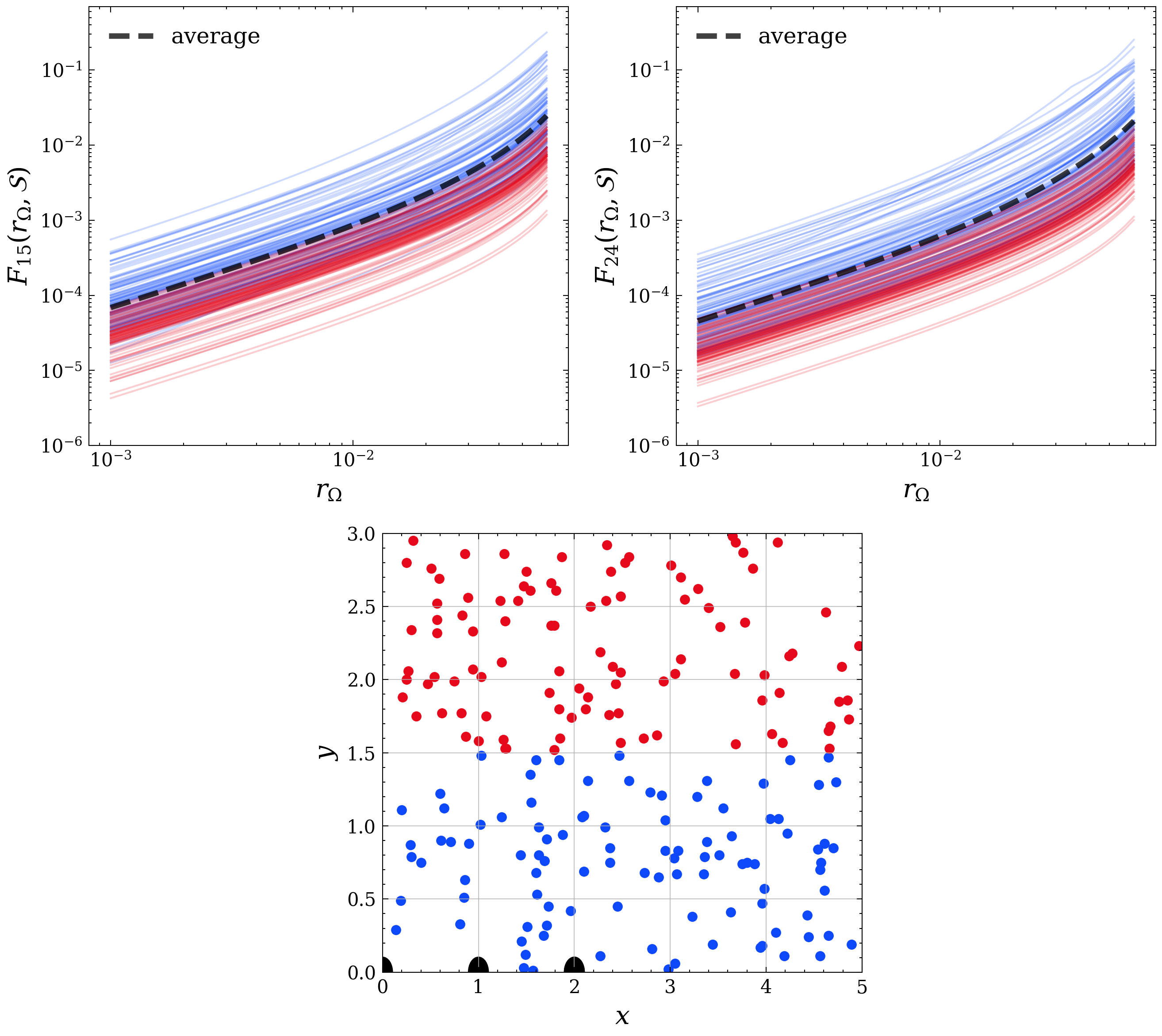}
\includegraphics[width=0.33\textwidth, trim=6cm 0cm 5cm 10cm, clip]{radius_many.png}
\caption{Perturbations of system $D$ by an intruder of variable radius sampled at $200$ locations. In the bottom plot, the sample points $\mathcal{S}$ are chosen uniformly at random over the rectangular domain $[0,5]\times[0,3]$ and colour coded according to their proximity to the resonators $D_3,D_4,D_5$ which are partially shown in black. The plots on top display the associated curves in $r_\Omega$ for the functions $F_{15}$ and $F_{24}$ with respect to $\mathcal{S}$.}
\label{fig: radpos1}
\end{figure}

\subsection{Varying position}\label{sec: varying position}

We now investigate the effect the position variable $\tilde{z}=(x,y)\in\R^2$ has on the function $F_{pq}$. First, let us fix the radius of the intruder to be the same as the resonators used for sensing, that is $r_\Omega=0.1$. The resulting function describes a three-dimensional surface parametrised by $(x,y,F^{0.1}_{pq}(x,y))$, where the third coordinate is the restriction $F^{0.1}_{pq}(x,y) := F_{pq}(0.1,x,y)$. To simulate this surface, we discretise a rectangular domain surrounding the resonators into a uniform grid of points. We use equal spacing along both axes, $\Delta x = \Delta y=0.01$. At each point we simulate the perturbed capacitance matrix and calculate its eigenvectors. We record how much the expected parity is broken via $F_{pq}^{0.1}$. The heatmaps for the surfaces $F_{15}^{0.1}$ and $F_{24}^{0.1}$ are shown in Figure~\ref{fig: positions1}. We have avoided overlapping with the resonators. This is indicated by the white annuli of radius $r_\Omega$ surrounding each resonator.

\begin{figure}
\centering
\begin{tikzpicture}
\node[inner sep=0pt] (a) at (-3,0)
    {\includegraphics[width=0.49\textwidth,trim=0 0 8.75cm 0cm,clip]{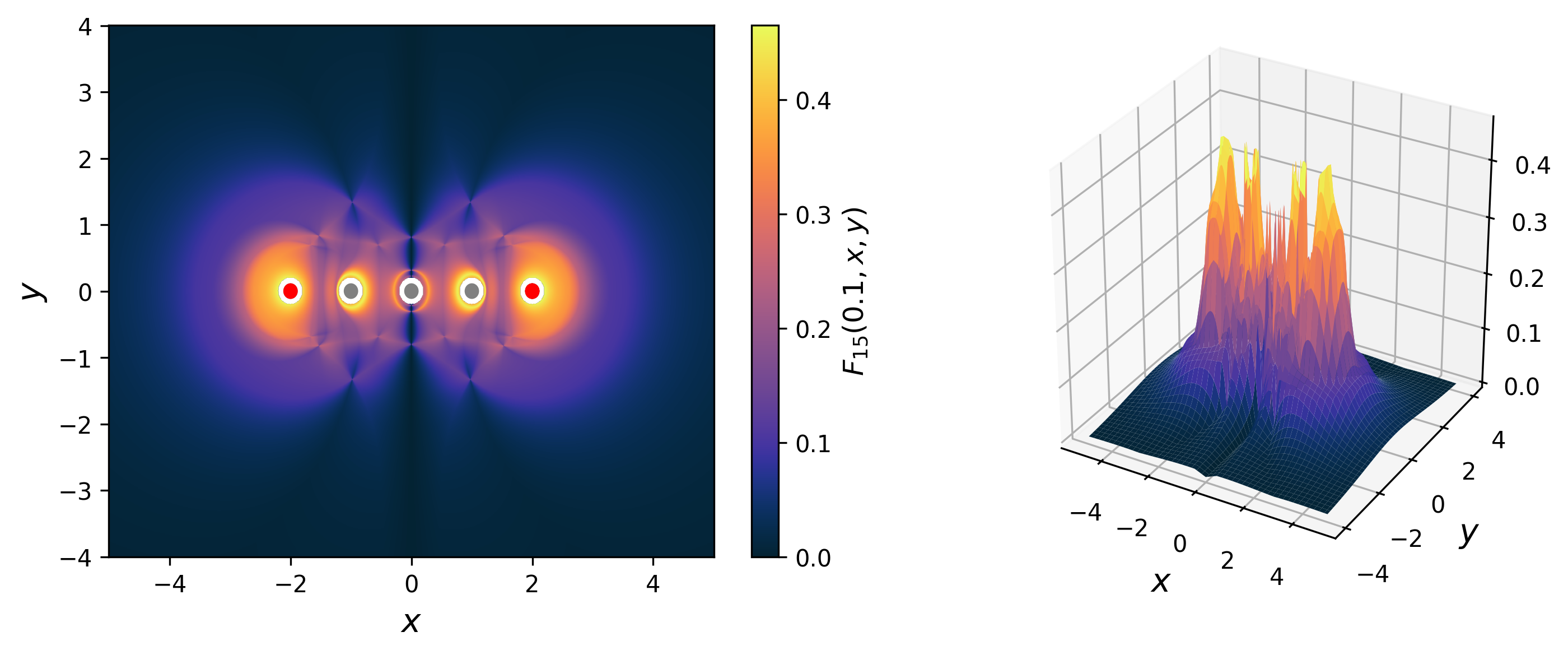}};
\node[inner sep=0pt] (b) at (3.5,0)
    {\includegraphics[width=0.45\textwidth,trim=0 0 10cm 0,clip]{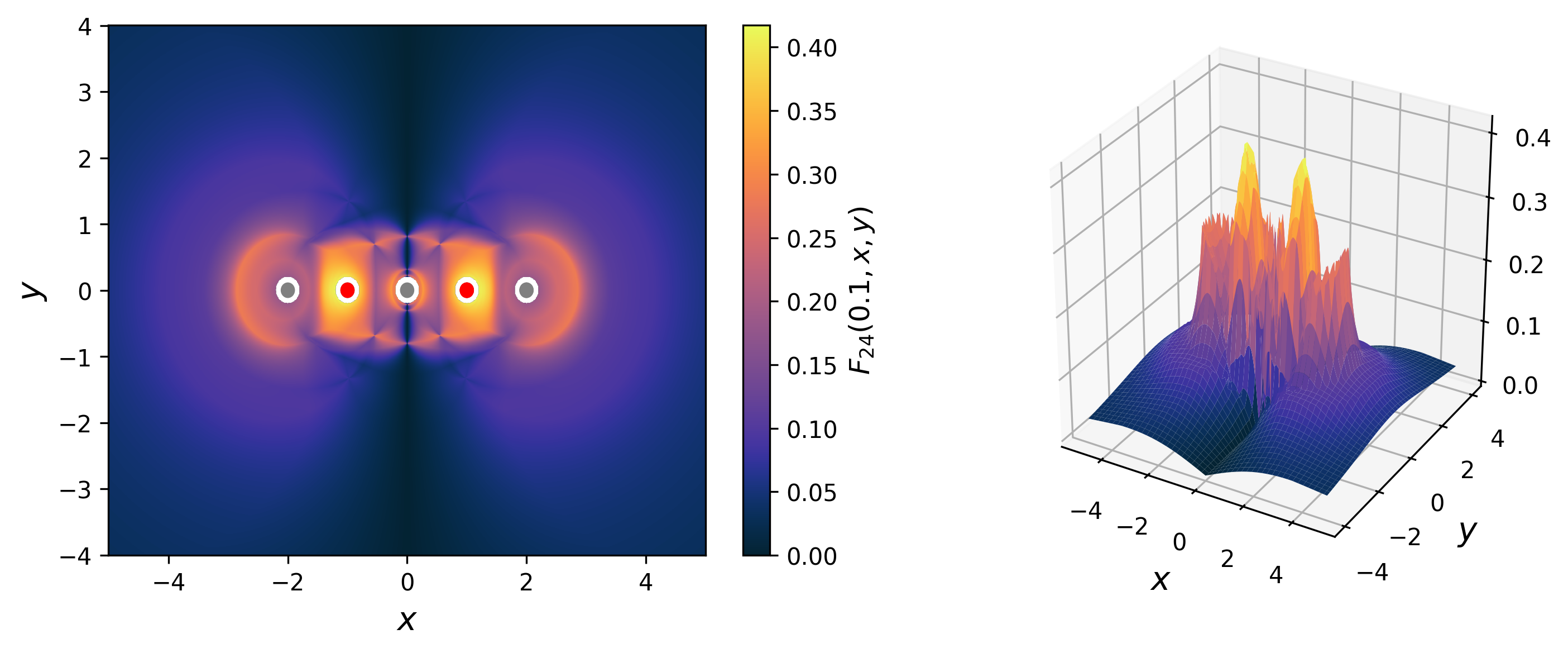}};
\node[scale=0.8,white] at (-5.2,2) {(a)};
\node[scale=0.8,white] at (1.3,2) {(b)};
\end{tikzpicture}
\caption{Perturbations of system $D$ by an intruder of fixed radius but variable position. (a) the heatmap associated with the function $F_{15}^{0.1}$ are presented. (b) the heatmap associated with $F_{24}^{0.1}$. The symmetric pair of resonators used for sensing is marked in red, and the remaining resonators in gray. The white annuli surrounding the resonators represent regions where an intruder cannot exist as it would overlap the resonators. }
\label{fig: positions1}
\end{figure}

There are a few interesting observations that can be made. Firstly, the heatmaps are symmetric along the lines of symmetry - the $x$\nobreakdash-axis and the $y$\nobreakdash-axis. This means that the inverse mapping would be multivalued with each input having at least four different outputs. One way to overcome this problem is to restrict to one of the four quadrants like we did in the previous section. Note further that if an intruder is located along the line $y=0$, it does not break the symmetry, meaning that $F_{pq}^{0.1}(x,0) = 0$ for all $x\in\R$, and we will not be able to detect it.

Secondly, we observe that as we move away from the resonators, the sensitivity decreases quickly, consistent with Proposition~\ref{prop: |(x,y)| tends to inf}. We can reinforce this observation by checking that far away $F^{0.1}_{pq}$ does indeed tend to zero. To do this, we choose two paths along which to simulate the asymptotic behaviour. We first fix $x_0=0.5$ and vary $y\in\R$ after which we swap their places and fix $y_0=0.5$ while varying $x\in\R$. Figure~\ref{fig: asymp_x_y} shows this. Interestingly, and perhaps not as obvious in the surface plots, the pair $\{D_2,D_4\}$ appears to be able to sense significantly further along the $x$\nobreakdash-axis than $\{D_1,D_5\}$.

\begin{figure}
\centering
\begin{tikzpicture}
\node[inner sep=0pt] (a) at (0,0)
    {\includegraphics[width=0.8\textwidth]{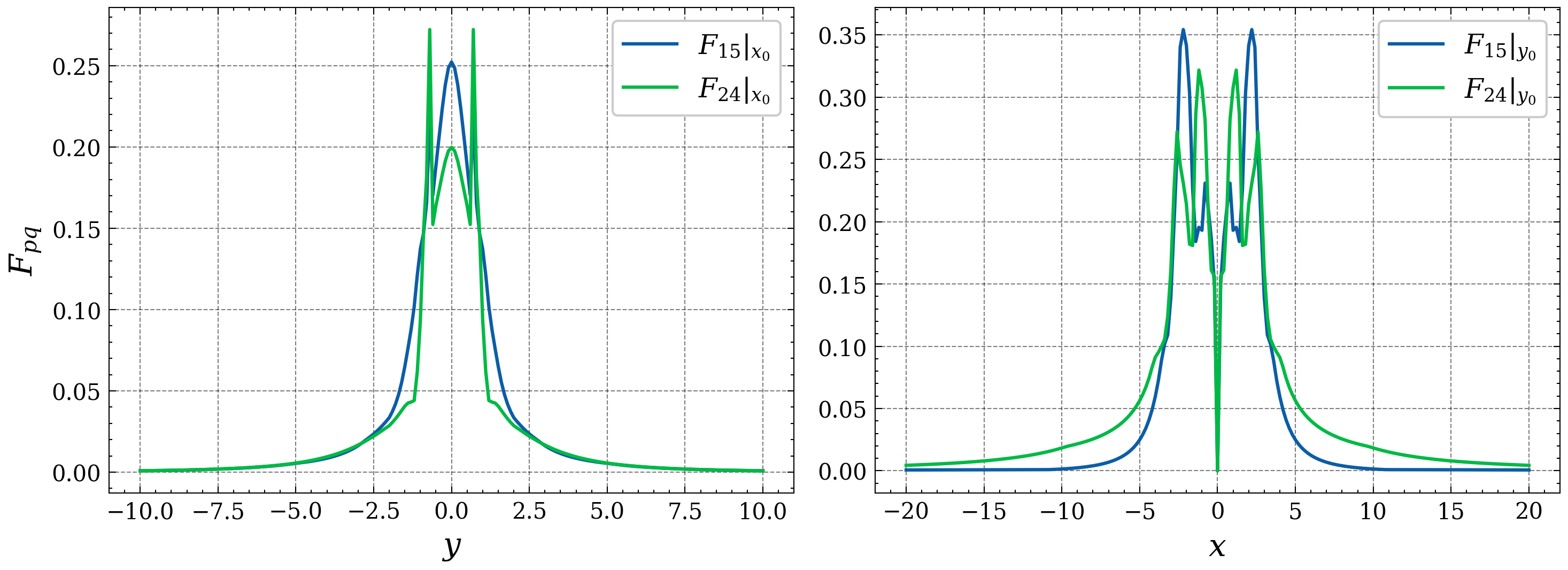}};
\node[scale=0.8] at (-5,1.8) {(a)};
\node[scale=0.8] at (1.1,1.8) {(b)};
\end{tikzpicture}
\caption{Far-field behaviour of the functions $F_{15}$ and $F_{24}$. (a) a path along $\{x_0\}\times[-10,10]$ with $x_0=0.5$ is simulated showing the asymptotic behaviour along the $y$-axis. (b) a path along $[-20,20]\times\{y_0\}$ with $y_0=0.5$ is simulated to show the asymptotic behaviour along the $x$-axis.}
\label{fig: asymp_x_y}
\end{figure}

The last thing worth pointing out is that near the resonators the surfaces increase rapidly and are very jagged and non-smooth. In such regions, it will be difficult to fit regression models for predictive purposes as small changes in $(x,y)$ lead to large and sudden changes in $F^{0.1}_{pq}(x,y)$. On the other hand, as we move far away from the resonators, the surfaces become flat, also making it difficult for prediction purposes since large changes in position lead to little change in $F^{0.1}_{pq}$. When considering the inverse in these regions, it will be ill-posed as then small changes in input $F_{pq}$ lead to large changes in $(x,y)$.

We present a few more heatmaps in Figure~\ref{fig: many_pos_plots} for this system of resonators, but for smaller intruder radius $r_\Omega$. As one could expect, intruders with smaller radii have a weaker, but still noticeable effect on the resonator eigenmodes. For $r_\Omega=0.06$ (upper panel), there is a strong response near the symmetric pairs with less coupling visible. For $r_\Omega=0.02$ (lower panel), the effect is even weaker and the resulting surfaces appear to have isolated halos around the resonators.

\begin{figure}
\centering
\includegraphics[width=0.8\textwidth]{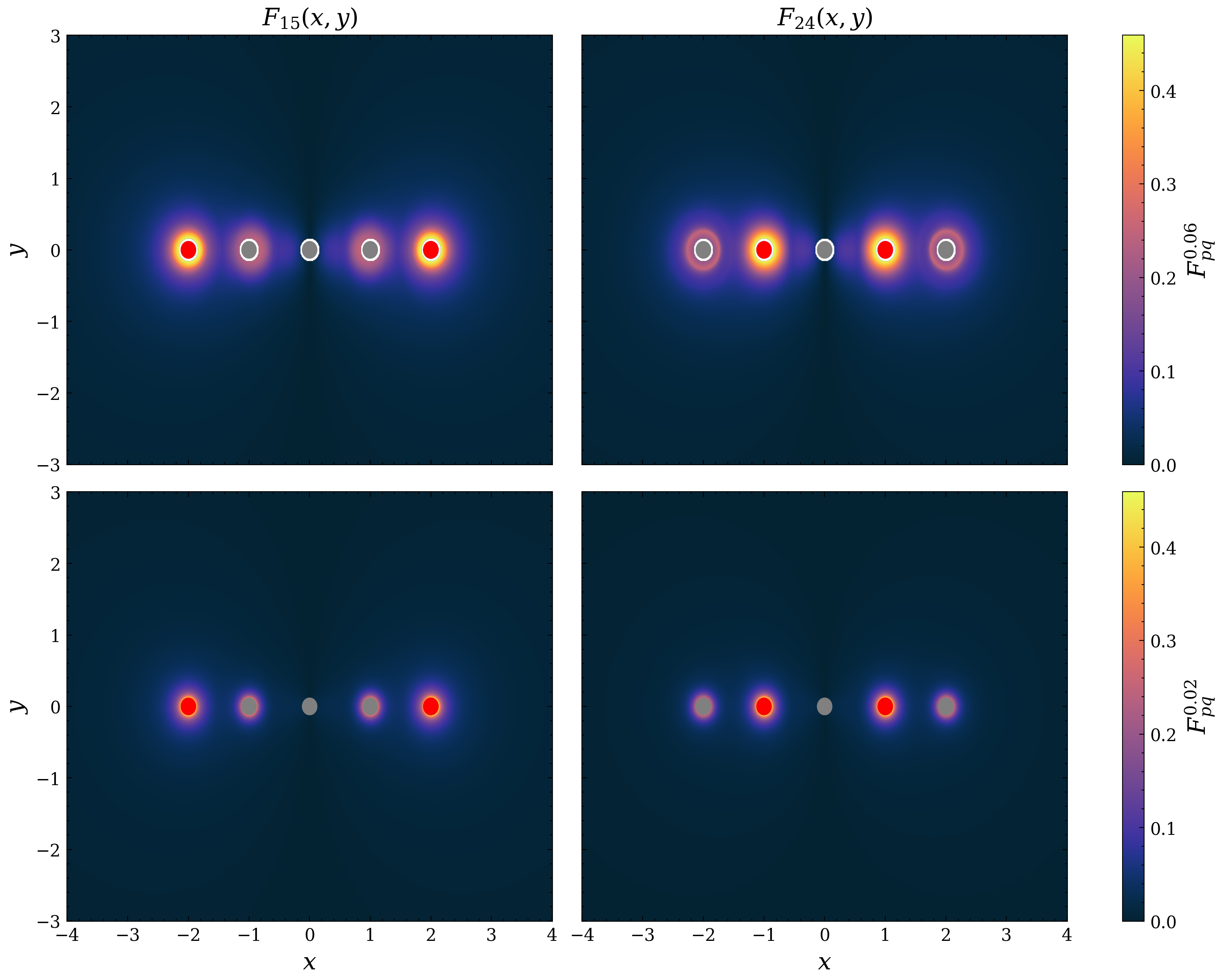}
\caption{Perturbations of system $D$ by an intruder of fixed radius but variable position. In the top row, heatmaps associated with the functions $F_{15}^{0.06}$ and $F_{24}^{0.06}$ are presented. In the bottom row are the heatmaps associated with $F_{15}^{0.02}$ and $F_{24}^{0.02}$. The symmetric pair of resonators used for sensing is marked in red, and the remaining resonators in gray. The white annuli surrounding the resonators represent regions where an intruder cannot exist as it would overlap the resonators.}
\label{fig: many_pos_plots}
\end{figure}

\subsection{Other resonator systems}
Having discussed some of the potential issues for the sensing functions $F_{15}$ and $F_{24}$ associated with the system of five collinear resonators, it is of interest to see if these issues persist in other systems. We present three new arrangements: one horseshoe shape with a single line of mirror symmetry, one fully asymmetric array and, finally, a sparse system with the resonators relatively far apart.


\subsubsection{Horseshoe arrangement of resonators}

The horseshoe system presented in Figure~\ref{fig: horseshoe_pos_many} shows that all the pairs are strongly correlated with each other. For example, looking at the resonators at the tips of the horseshoe, they pick up a strong signal from the intruder when it is located near one of the other resonators. This happens for all three symmetric pairs. We also check that the varying radii are consistent with the previous analysis by once again taking a uniformly random sample of $200$ points over the domain $[-4,5.5]\times[-5,5]$. As usual we avoid sampling points that would lead to overlapping with the resonators. Figure~\ref{fig:many_rad}(a) reinforces our hope that the radius curves are parallel.

\begin{figure}
\centering
\includegraphics[width=0.9\textwidth]{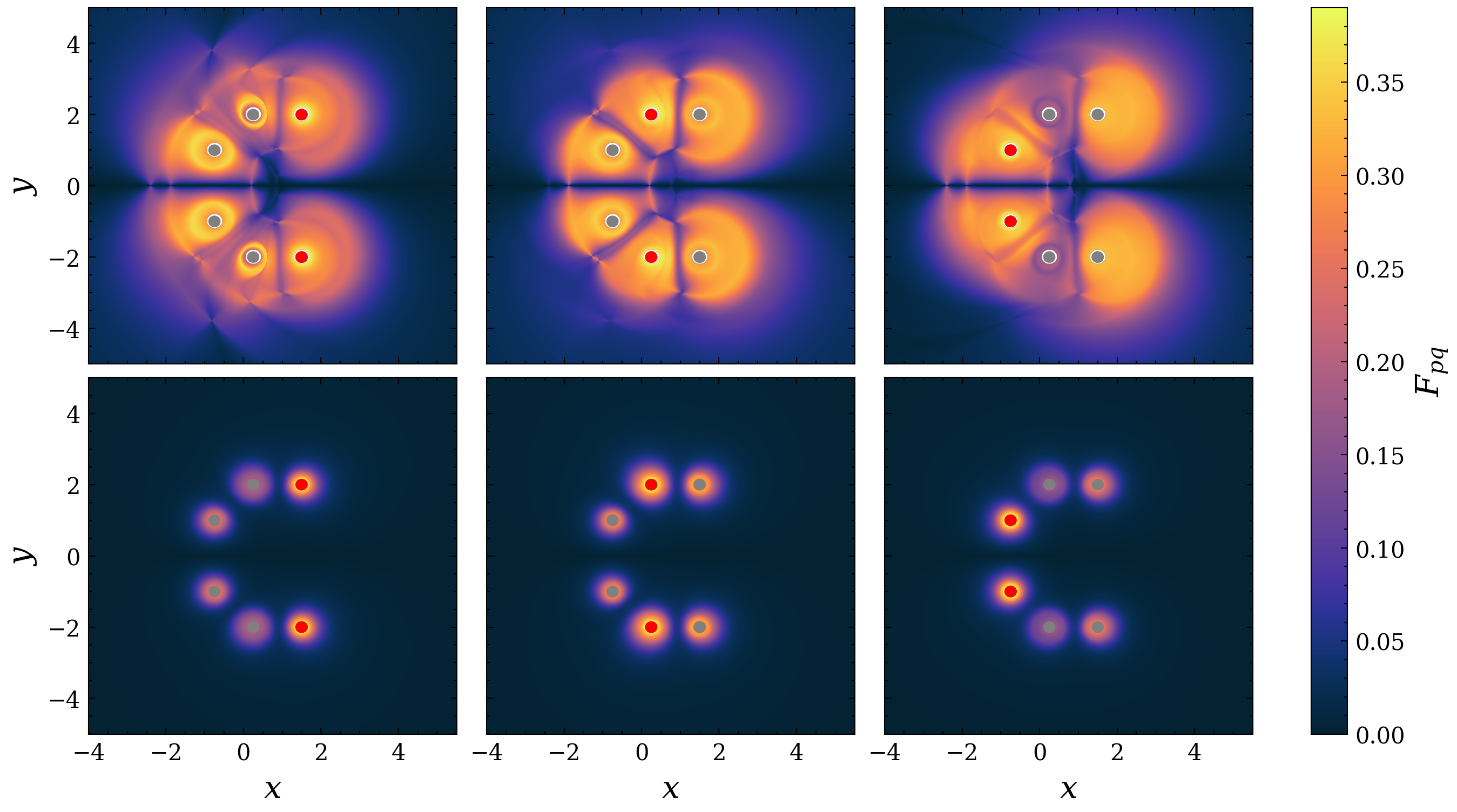}
\caption{Perturbations of a horseshoe system by an intruder of fixed radius but variable position. In the top row, heatmaps associated with intruder radius $r_\Omega=0.1$ are presented, while in the bottom row, heatmaps associated with $r_\Omega=0.05$ are presented. The symmetric pair of resonators used for sensing is marked in red, and the remaining resonators in gray. The setup is mirror symmetric with respect to the $y=0$ line.}
\label{fig: horseshoe_pos_many}
\end{figure}

\begin{figure}
\centering
\begin{tikzpicture}
    \node[inner sep=0pt] (horseshoe) at (0,0)
    {\includegraphics[width=.31\textwidth]{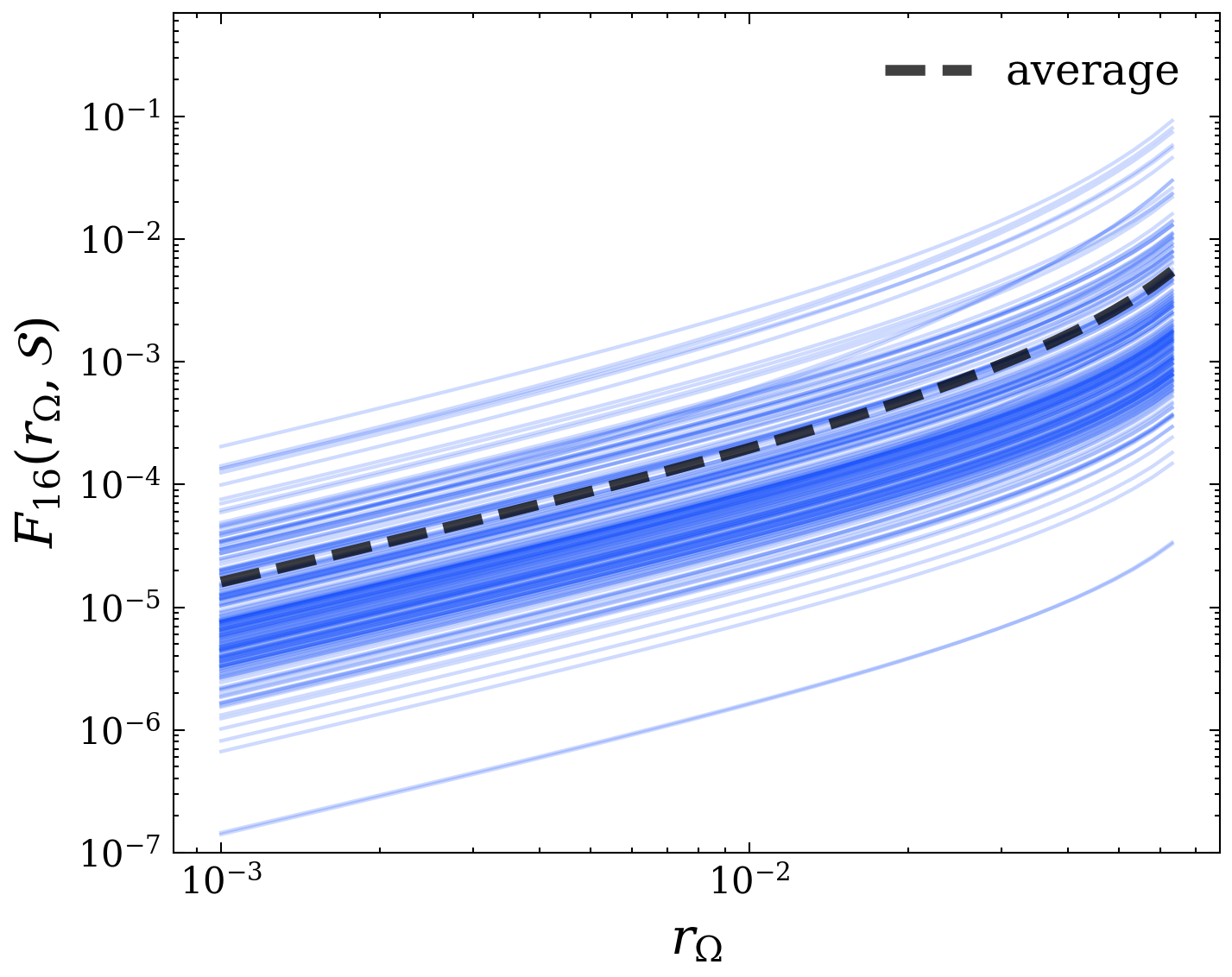}};
    \node[scale=0.8] at (2,-1.1) {(a)};
    \node[inner sep=0pt] (latsym) at (5.1,0)
    {\includegraphics[width=.31\textwidth]{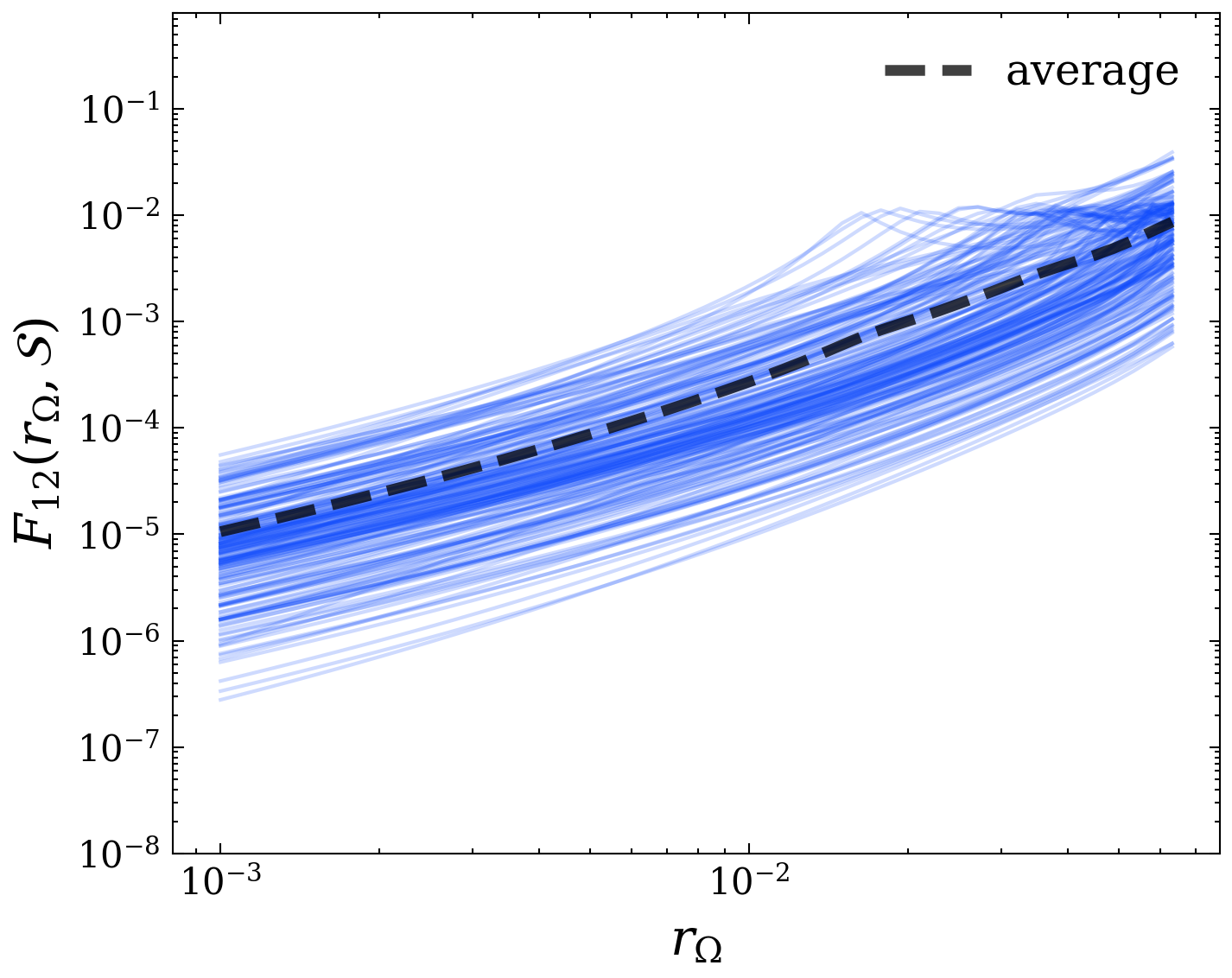}};
    \node[scale=0.8] at (5.1+2,-1.1) {(b)};
    \node[inner sep=0pt] (sparse) at (10.2,0)
    {\includegraphics[width=.31\textwidth]{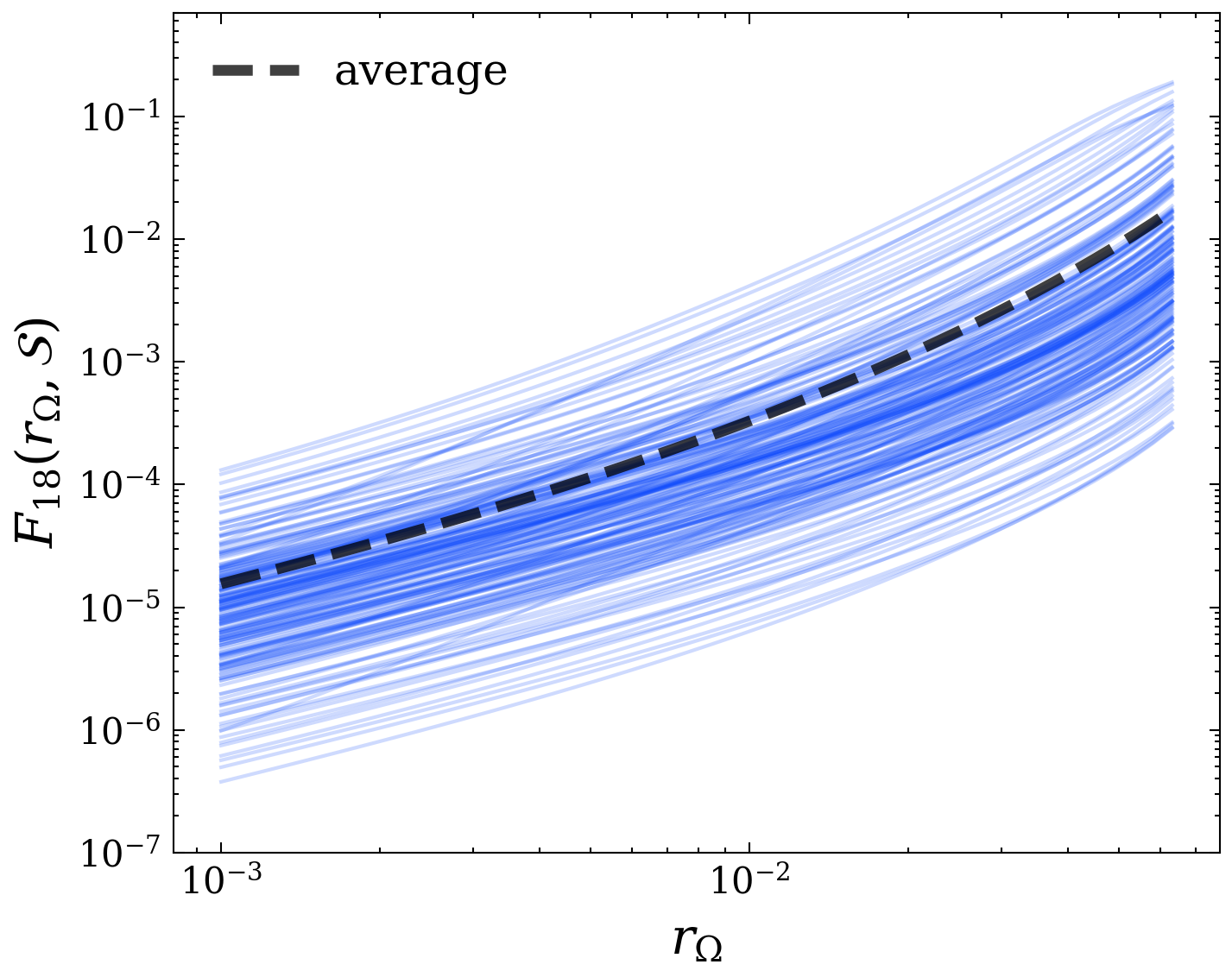}};
    \node[scale=0.8] at (10.2+2,-1.1) {(c)};
\end{tikzpicture}
\caption{Perturbations of a resonator system by an intruder of variable radius sampled at $200$ locations. The sample points $\mathcal{S}$ are chosen uniformly at random over a rectangular domain $\mathcal{R}$, such that they don't overlap with the resonators. (a) shows the horseshoe system, with $\mathcal{R}=[-4,5.5]\times[-5,5]$. The function $F_{16}$ corresponds to the symmetric pair located at the tips of the horseshoe. (b) shows the asymmetric array from Figure~\ref{fig: latsym_pos}, with $\mathcal{R}=[-7,7]\times[-7,7]$. (c) shows the sparse system from Figure~\ref{fig: sparse_pos_many}, with $\mathcal{R}=[0,20]\times[-10,10]$.}
\label{fig:many_rad}
\end{figure}
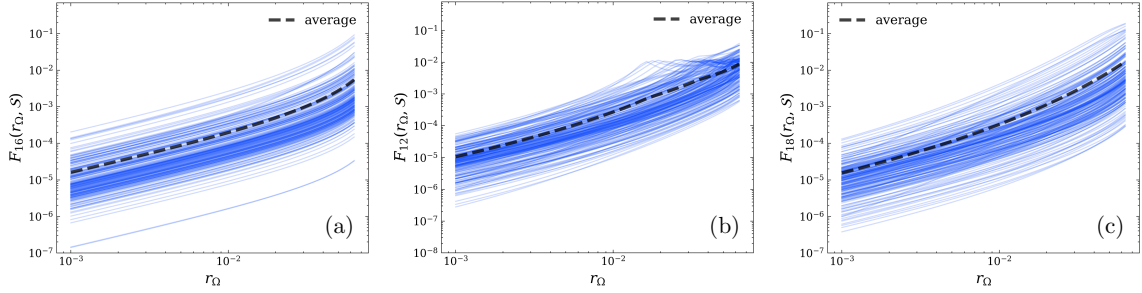

\subsubsection{Asymmetric arrangement with latent symmetry}

The array shown in Figure~\ref{fig: latsym_pos} is one of the smallest systems that contain a latent symmetry. It consists of ten resonators and only one pair is symmetric. An advantage of using asymmetric systems is that they are able to distinguish between all of the domain. Something interesting to note about this system is that the sensing pair appears to only be strongly correlated with a subset of the system. The radius curves in Figure~\ref{fig:many_rad}(b) appear to be parallel again with the exception of some displaying a hump towards larger radii. This is likely due to those curves corresponding to locations located in between the resonators where stronger interference effects occur for larger radii.

\begin{figure}
\centering
\includegraphics[width=0.8\textwidth]{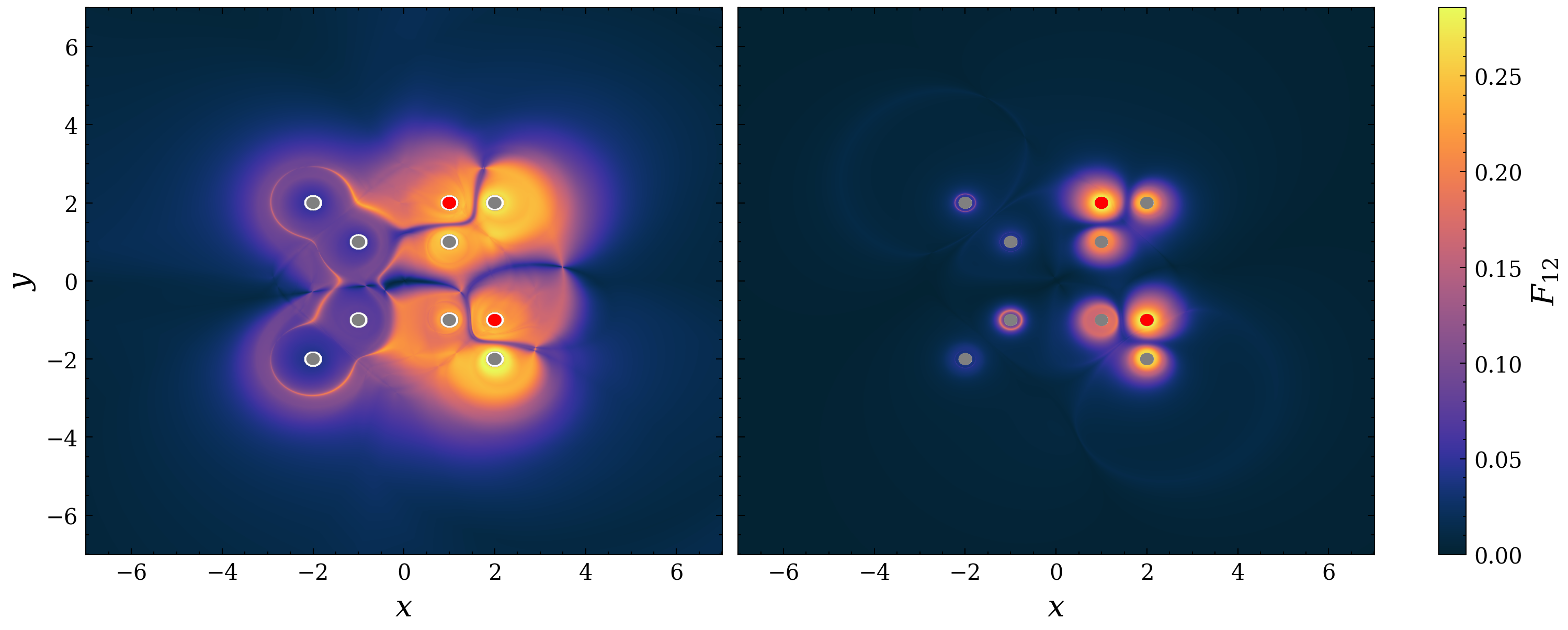}
\caption{Perturbations of an asymmetric but latently symmetric system by an intruder of fixed radius but variable position. The heatmap on the left corresponds to an intruder of radius $r_\Omega=0.1$, while the heatmap on the right corresponds to $r_\Omega=0.05$. The symmetric pair of resonators used for sensing is marked in red, and the remaining resonators in gray. The distinguishable feature of these systems with latent symmetry is that there are no mirror symmetries.}
\label{fig: latsym_pos}
\end{figure}

\subsubsection{Sparse arrangement of resonators}

The sparse system in Figure~\ref{fig: sparse_pos_many} consists of 9 resonators, arranged as four symmetric pairs formed by a square reflected in the line $x=0$, and a central resonator that breaks some of the mirror symmetry. The coordinates of the resonator pairs are $\{(5,2), (-5,2)\}$, $\{(1,2),(-1,2)\}$, $\{(5,-2),(-5,-2)\}$ and $\{(1,-2),(-1,-2)\}$ and the central resonator is located at $(0,2)$. Figure~\ref{fig: sparse_pos_many} shows the effect the intruder's location has on the system. Since the system is symmetric, we plot only the right-hand side. We can see that some resonator pairs, that is, $\{(5,2), (-5,2)\}$ and $\{(5,-2),(-5,-2)\}$, pick up signals far away, while the other two pairs are only able to feel the nearby region. The pair $\{(1,-2),(-1,-2)\}$ seems to be able to sense relatively far as well for intruder radius of $r_\Omega=0.1$. This overall behaviour, where the sensing pairs complement each other is promising for detecting intruders that are both near and far. Furthermore, the radius curves in Figure~\ref{fig:many_rad}(c) are once again parallel and suggest that the functions $F_{pq}$ are separable into a product of a radial function and a function that depends on the position variable only, regardless of the system setup.

\begin{figure}
\centering
\includegraphics[width=\textwidth]{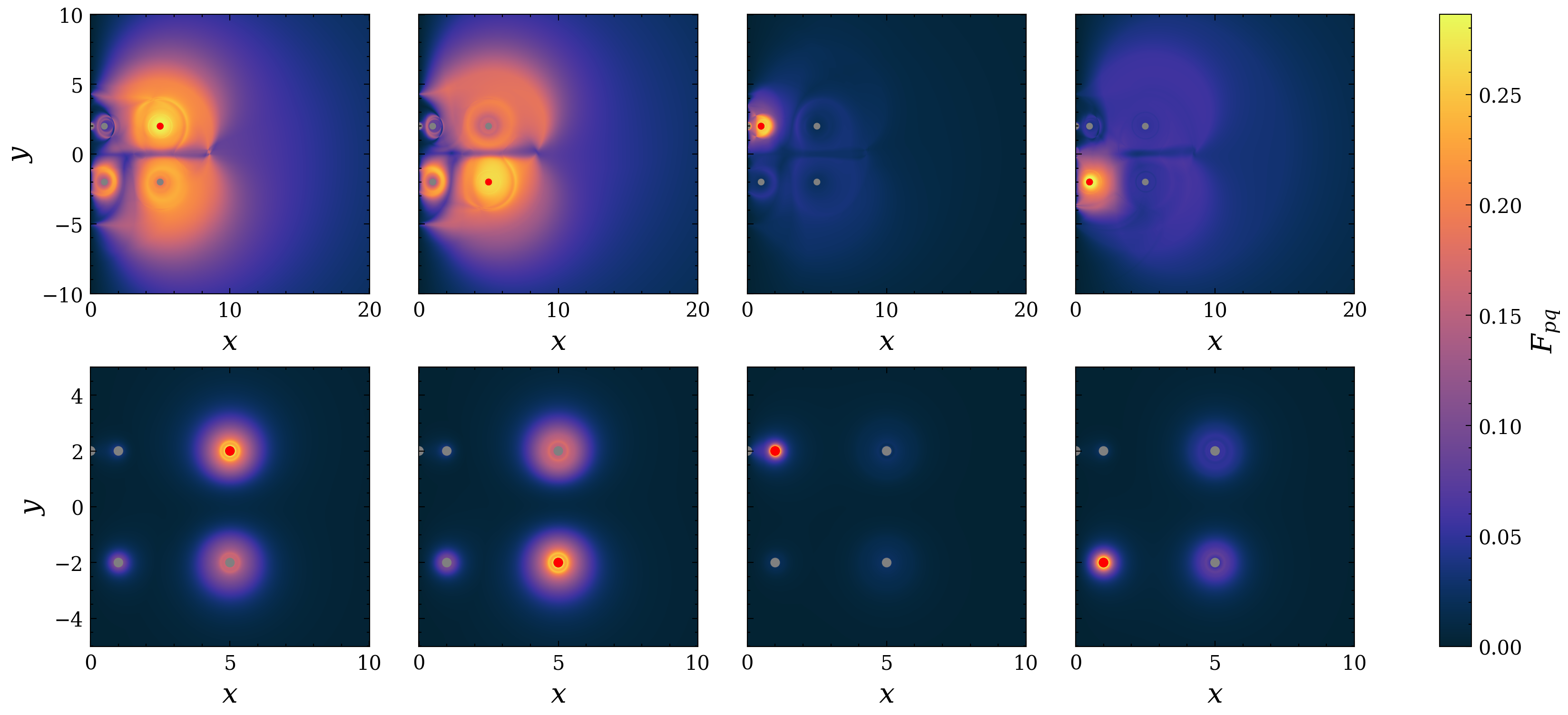}
\caption{Perturbations of a sparse system by an intruder of fixed radius but variable position. The heatmaps in the top row correspond to an intruder of radius $r_\Omega=0.1$, while the heatmaps in the bottom row correspond to $r_\Omega=0.05$. The resonator corresponding to $F_{pq}$ is highlighted. Its pair is not shown, as we plot only half of the symmetric array here (the mirror symmetry is along the line $x=0$). There is a resonator at coordinates $(0,2)$ which breaks the mirror symmetry along the line $y=0$. Note the difference in scale between the two rows and the overall scale of these heatmaps. The pairs further away from the $x=0$ line of symmetry exhibit strong sensing far away from the resonators.}
\label{fig: sparse_pos_many}
\end{figure}

%% file: 4_Detection.tex
\section{Target identification and localization} \label{sec:sensing}

In section~\ref{ch: pertResSys}, we studied the forward problem, to understand the behaviour of the functions $F_{pq}$ through which we measure perturbations caused by an intruder to a system of symmetric resonators. We are now ready to tackle the inverse problem: given observations of $F_{pq}$, predict the position and radius of the intruder. Predicting the radius given an already known location for the intruder $\Omega$ is relatively easy since it is a one-dimensional regression problem. However, predicting the location of $\Omega$ is much harder due to the issues we discussed in section~\ref{sec: varying position} about multivaluedness and ill-conditioning\footnote{By ill-conditioning we refer to the idea that the output data is sensitive to small changes in input data and vice-versa.} in certain domains. Hence, the majority of our effort will be devoted to predicting the position of an intruder.

\subsection{Predicting radius}\label{sec: predicting radius}
Recall the setup consisting of five collinear resonators and an intruder located at $(3,0)$, as depicted in Figure~\ref{fig: 5 chain preturbed}. We showed that for the intruder's radii up to around $r_\Omega=0.08$ both $F_{15}$ and $F_{24}$ follow similar trajectories, so we focus only on $F_{15}$. Throughout this section we restrict our attention to intruders with $r_\Omega \le 0.08$ because our sensing mechanism is based on measuring the breaking of a latent symmetry and, as is typical of perturbative methods, this is most effective when the perturbation is small. As shown in Figure \ref{fig: radius1}, when $r_\Omega$ approaches the radius of the array's resonators the functions $F_{pq}$ become oscillatory and no longer determine $r_\Omega$ uniquely.

We saw the presence of upward curvature in the trajectory in the log-log scale of Figure~\ref{fig: radius1}, indicating growth faster than power laws. A natural guess would be to consider exponential-type behaviour. In particular, we fit a regression model by tuning the parameters of the guess function $f(r;\theta) = a\exp{(br^c)}$, where $\theta=(a,b,c)$. We discretise a domain for each parameter and find the optimal values such that the mean-squared loss function
\begin{equation}
    \mathcal{L}(\theta) = \frac{1}{n}\sum_{i=1}^n\left|f(r_i;\theta) - F_{15}(r_i)\right|^2
\end{equation}
is minimised. Here, $r_1,...,r_n$ denote radius samples for $r_\Omega$. We use the \texttt{curve\_fit} function from the \texttt{scipy.optimise} Python package to minimise the loss function. The underlying optimisation algorithm used is a combination of gradient descent and the Gauss-Newton method. The optimised parameters are found to be $\theta_{\text{opt}}=(9.31\times10^{-4},5.589,0.826)$. Figure~\ref{fig: fitted_radius} shows how this compares with the true function $F_{15}(r_\Omega,3,0)$. The pointwise absolute error $\text{err}_f = |F_{15}(r_i)-f(r_i;\theta_\text{opt})|$ is also plotted showing that this exponential guess function is within three to five decimal point precision from the true function.

\begin{figure}
\centering
\includegraphics[width=0.7\textwidth]{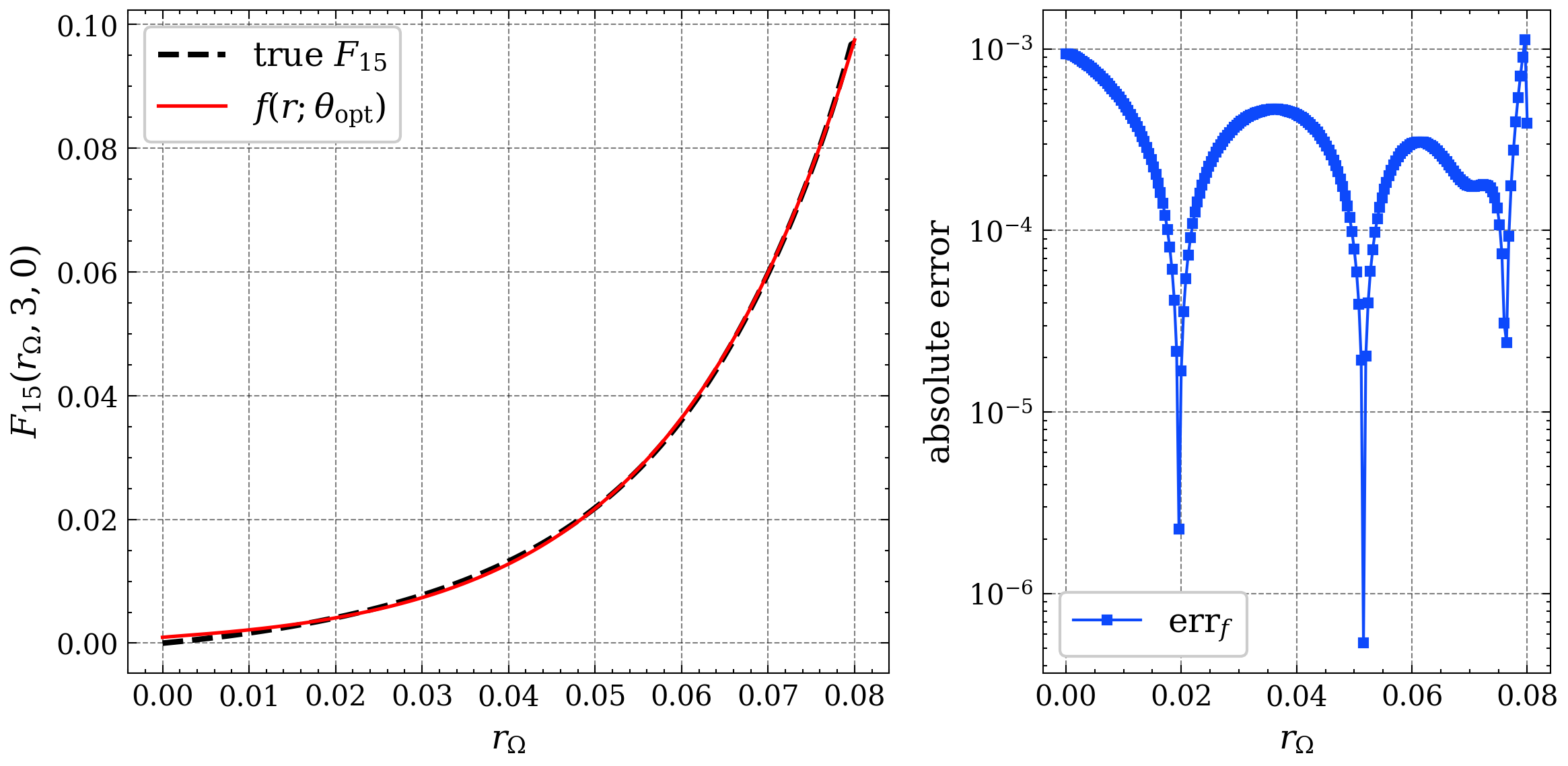}
\caption{On the left, a guess function of the form $f(r;\theta)=a\exp(br^c)$ is fitted to the true model $F_{15}(r,3,0)$. After optimisation the optimal values for $\theta$ are $a=9.31\times10^{-4}$, $b=5.589$, and $c=0.826$. The error between the true model and the predictive model is plotted on the right, and it shows accuracy ranging between three to five significant figures.}
\label{fig: fitted_radius}
\end{figure}

In practical applications, measurement noise needs to be considered. To explore this, it is typical to assume multiplicative noise in sensing problems where the amount of noise scales with the strength of the response. Hence, let us assume the actual observations we make are modelled by $F_\text{obs}^{i} = F_{15}(r_i)(1+\eta)$ with $\eta\sim\mathcal{N}(0,\sigma^2)$. Suppose further that $\sigma=0.2$, a deviation of $20\%$ from the true value. A trick to re-model this into a framework with additive noise is by taking logarithms, to give
\begin{equation}
    \log{F^i_\text{obs} \approx} \log{F_{15}}(r_i) + \eta,
    \label{eq: log_noise}
\end{equation}
We can now construct a Gaussian Process Regression (GPR) model, taking inspiration from previous successful applications to inverse scattering problems \cite{bai2024gaussian, sung2024functional}. To demonstrate the practicality of GPR in this problem, we train the model on only ten noisy samples corresponding to ten evenly distributed radii from the interval $[0,0.08]$. We use the squared exponential kernel for the covariance function, and we perform a Maximum a Posteriori (MAP) approximation on the length scale and amplitude hyperparameters. This is done by minimising the negative log-likelihood function. The \verb|ScipyMinimize| function from the \texttt{JAX}\nobreakdash-based optimisation package \texttt{jaxopt} was used as it incorporates automatic differentiation with the BFGS algorithm\footnote{The Broyden-Fletcher-Goldfarb-Shanno (BFGS) algorithm is a quasi-Newton method used in solving unconstrained nonlinear optimisation problems \cite{NAYAK2020135}.} for unconstrained nonlinear optimisation problems. Finally, we condition the model on a discretisation of $[0,0.08]$ into two hundred evenly distributed points.

\begin{figure}
\centering
\includegraphics[width=0.7\textwidth]{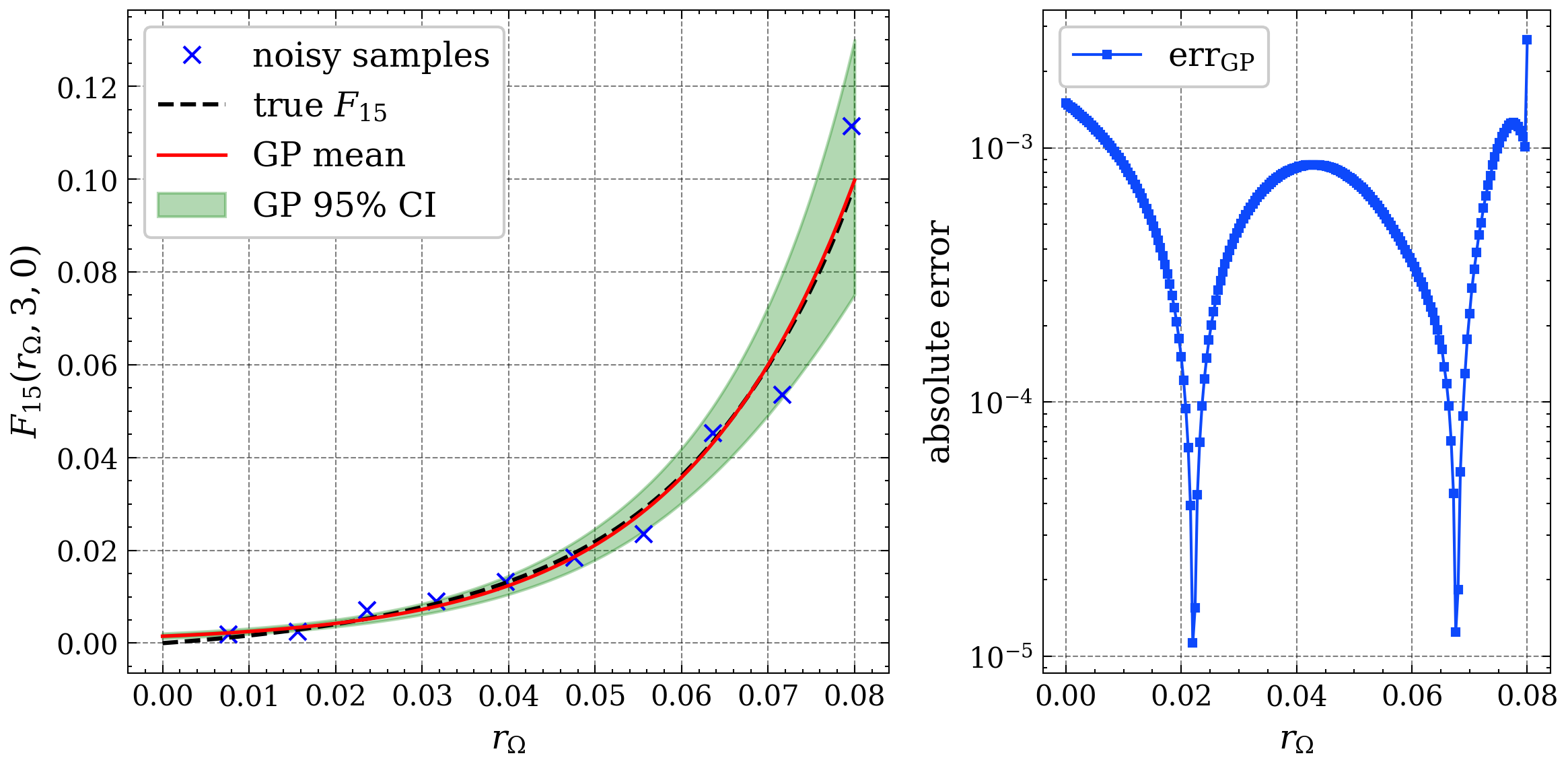}
\caption{Gaussian process regression model for the true function $F_{15}(r_\Omega,3,0)$. The process is fitted to $10$ samples of multiplicative noise from a normal distribution with mean $m=1$ and variance $\sigma^2=0.04$. After conditioning, the mean Gaussian process function is plotted in red, and a $95\%$ confidence region is plotted around it in green.}
\label{fig: pred_noisy_radius}
\end{figure}

The posterior log-mean $\tilde\mu$ and associated variance $K_*$ form our predictive model and related uncertainty, respectively. The final thing left to do is to transform back to the multiplicative noise framework. The caveat is to note that for a random variable $\xi\sim\mathcal{N}(m,\sigma^2)$
\begin{equation}
    \mathbb{E}[\exp(\xi)] = \exp(m+\tfrac{1}{2}\sigma^2).
\end{equation}
Therefore, when we exponentiate, the predictive mean becomes $\mu_*(r_i) = \exp(\tilde\mu(r_i) + K_*(r_i))$. Figure~\ref{fig: pred_noisy_radius} visualises how the model smooths out the noise and how the predicted curve closely follows the true function. Another benefit of GPR is that we can quantify the uncertainty in the prediction, represented by the green $95\%$ confidence interval. We see that the uncertainty increases as the radius gets larger which is expected since the amount of noise is proportional to the magnitude of the perturbation.

If one was to assume $f(r;\theta_\text{opt})$ represents the true function and multiplicative noise was applied to it, then the noisy model would be governed by $f(r;\theta_\text{opt})(1+\eta)$. However, as shown in Figure~\ref{fig: pred_noisy_errors}, it becomes highly oscillatory with larger radii, and due to its non-bijectivity, it will be impossible to obtain a unique value for $r_\Omega$ given an observation of $F_{15}$. Notice that this does not happen with the GPR model. Furthermore, we can quantify the accuracy of the two models in the presence of noise by calculating the pointwise absolute errors
\begin{equation}
    \text{err}_{\tilde{f}} (r_i)=|F_{15}(r_i) - \tilde{f}(r_i;\theta_\text{opt})|, \quad \text{err}_\text{GP} (r_i)=|F_{15}(r_i) - \mu_*(r_i)|,
\end{equation}
where $\tilde{f}=(1+\eta)f$ denotes the noisy model. The corresponding results are shown in Figure~\ref{fig: pred_noisy_errors}. We see that the accuracy of the GPR model remains consistently low as $r_\Omega$ increases with occasional dips where the trajectory of $\mu_*$ crosses the true function. On the other hand, we can see an overall increase in the error for $f$ which is expected. It is worth mentioning that one can improve the accuracy of the GPR model even further by increasing its training set. 

\begin{figure}
\centering\includegraphics[width=0.9\textwidth]{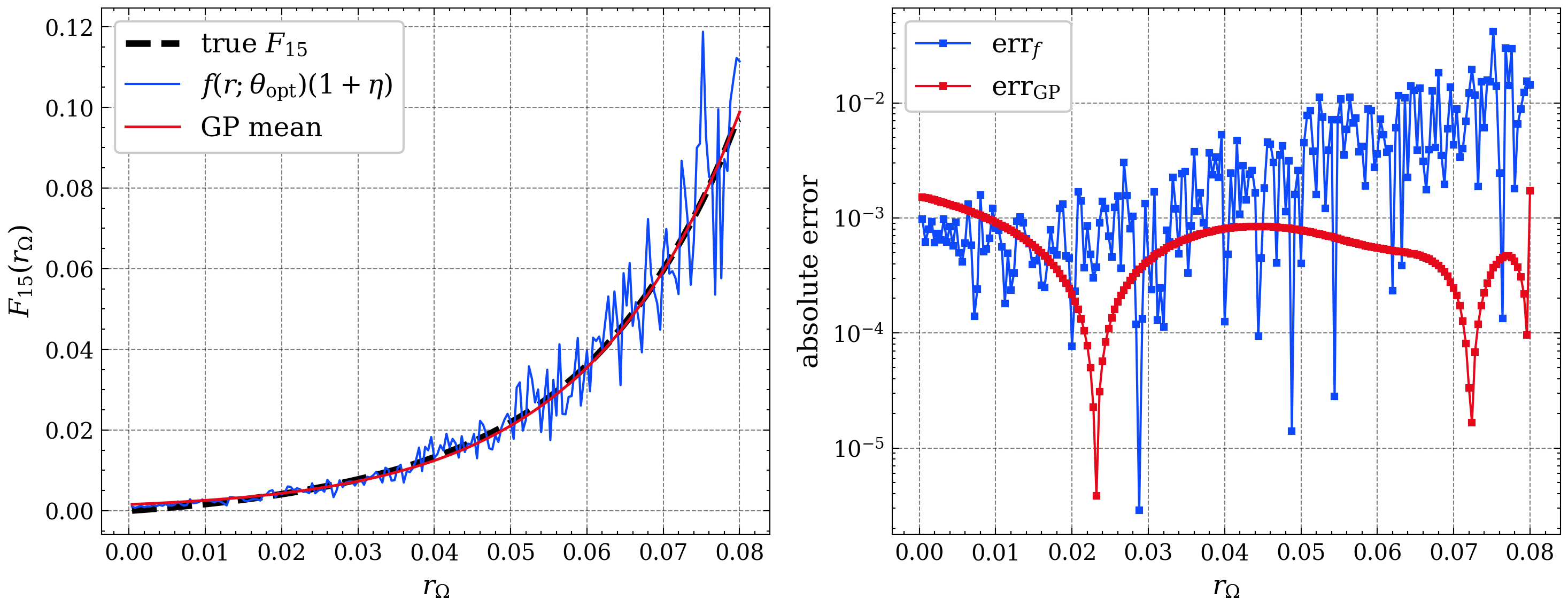}
\caption{Comparison plots for the noisy guess function $f(r;\theta_\text{opt})(1+\eta)$ with $\eta\sim\mathcal{N}(0,0.04)$, and the Gaussian process regression model. On the left, a sample of the noisy guess function is plotted, as well as the mean function of the conditioned Gaussian process. On the right, the pointwise absolute errors for both models are shown.}
\label{fig: pred_noisy_errors}
\end{figure}

Lastly, in section~\ref{sec: varying radius}, we showed that at different intruder locations, the $\boldsymbol{F}$ functions were essentially the same up to a multiplicative constant. Therefore, if we know the position of the intruder, the exponential guess function and the GPR model are appropriate predictive methods for the deterministic outputs and outputs affected by measurement noise, respectively.

\subsection{Predicting position}\label{sec: predicting position}

Similarly to the previous section, we will construct models for both deterministic and noisy observations of $F_{pq}$ to predict the intruder's position. We consider two different approaches with the noisy case: a multi-layer perceptron (MLP) and a Bayesian inference model. Once again, we will start with the simple system of five resonators with radii of $0.1$ arranged in a line. We assume the radius of the intruder $\Omega$ is fixed to $r_\Omega=0.06$. As we saw in the previous chapter, the surfaces representing the perturbation functions $F_{pq}^{0.06}(x,y)$ exhibited steep and flat regions, sharp fluctuations near the resonators, and a kink along the line $y=0$. This makes it extremely difficult to fit a global three-dimensional function which would parametrise these surfaces, and so, we will instead consider the inverse problem directly. Now, our inputs are observations $\boldsymbol{f}^{pq} = (F_{pq}(x_i,y_i))_{i=1}^n$ and the outputs are position coordinates $(x,y)\in\R^2$. We know that a true inverse function does not exist due to the issues with multivaluedness and ill-conditioning, but we represent its preimage symbolically by
\begin{equation}
    F^{-1}_{pq}(z) = \{(x,y)\in\R^2:F_{pq}(x,y) = z\},
\end{equation}
and thus, when applied to an observation, $F^{-1}(\boldsymbol{f}^{pq}_i)$ denotes our set, ideally of size one, of positions that gave rise to the observation.

The system of five resonators has two lines of symmetry which make $F_{15}$ and $F_{24}$ symmetric in four regions. We will restrict ourselves to just the top right quadrant, noting that this does not affect the practicality of the setup since it is invariant to translations and rotations.

\subsubsection{Dictionary-based model}\label{sec: dictionary model}

For the deterministic case, when there is no noise in the data, we propose a simple dictionary-based matching method. We create a grid of points in our domain of interest, which for the purposes of demonstration will be $[0,4]\times[0,3]$, and we use precision to two decimal places. That is, the descretised domain is $\{0.01,0.02,...,3.99,4.00\}\times\{0.01,0.02,...,2.99,3.00\}$. We measure $F_{pq}$ at each grid point and store the tuple of three values. Given an observation $f^{pq}$ we can locate the closest value in the dictionary for $F_{pq}$ and output the corresponding position.

Figure~\ref{fig: pred_positions_min} shows a sample for an intruder's position and its dictionary matches. Not surprisingly, we do not get a unique prediction when we use only information from $F_{15}$ or $F_{24}$. After all, mappings from $\R^2$ to $\R$ cannot have continuous inverses\footnote{If interested, see Brouwer's Invariance of Domain Theorem.}, hence these surface functions have level sets. The predicted positions in Figures~\ref{fig: pred_positions_min}(a) and~\ref{fig: pred_positions_min}(b) follow these level sets.

One way to improve the model is to combine information from both symmetric resonator pairs instead of treating them separately. Given observations $f^{15}$ and $f^{24}$, we consider the loss function
\begin{equation} \label{eq:Fcombined}
    F_\text{loss}(x,y) =\left|F_{15}(x,y)-f^{15}\right| + \left|F_{24}(x,y)-f^{24}\right|,
\end{equation}
which is equivalent to the mean absolute error (MAE) function, and minimise it over our set of grid points. This significantly improves our model's predictions, as demonstrated in Figure~\ref{fig: pred_positions_min}(c). We get a single prediction which matches the true location exactly. As we will see later, this happens for any position sampled on our grid. The reason is that we now technically have a mapping from $\R^2$ to $\R^2$ and the existence of unique points is possible. The contour lines of $F_\text{loss}$, for the particular sample point shown, cluster towards a unique minimum, which is the prediction.

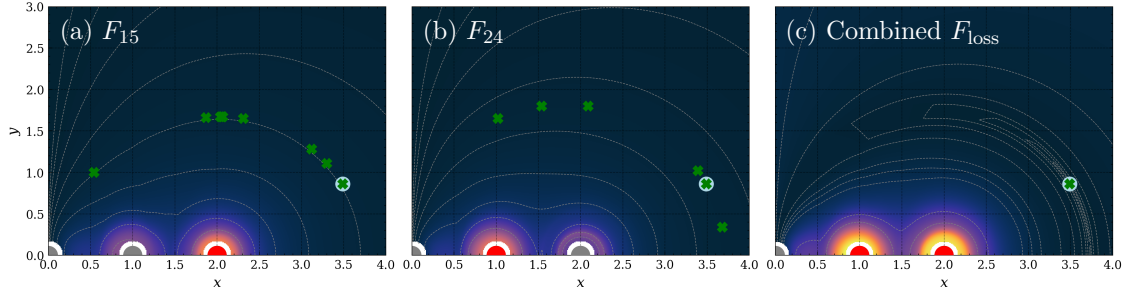
\begin{figure}
\centering
\begin{tikzpicture}
    \node at (0,0) (plots) {\includegraphics[width=0.95\textwidth]{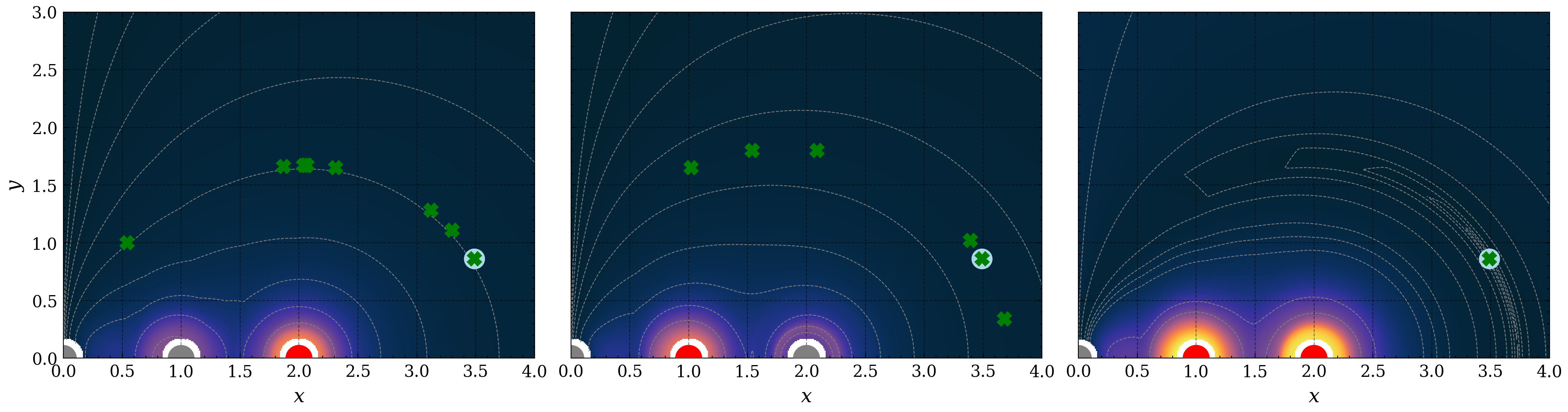}};
    \node[white,scale=0.9,anchor=west] at (-6.8,1.5) {(a) $F_{15}$};
    \node[white,scale=0.9,anchor=west] at (-2,1.5) {(b) $F_{24}$};
    \node[white,scale=0.9,anchor=west] at (2.8,1.5) {(c) Combined $F_\text{loss}$};
\end{tikzpicture}
\caption{Prediction of intruder's position using a dictionary-based model. (a) uses just $F_{15}$, (b) uses just $F_{24}$ and (c) uses the combined function $F_{\text{loss}}$ from \eqref{eq:Fcombined}. The true location is labeled with a white circle, while the predictions are given by green crosses. The intruder is of fixed radius $r_\Omega=0.06$. Level sets are indicated by the white curves. 
}
\label{fig: pred_positions_min}
\end{figure}

The improved model performs with exceptional accuracy in the absence of noise. It is also not very computationally heavy given that a heatmap needs to be created only once, and then locating the closest match is quick. For reference, the implementation was done using Python's \texttt{numpy} library which allows vectorisation of arithmetic, and it takes less than one second to locate the closest match and ten seconds to simulate a heatmap of $400\times300$ pixels on a typical laptop processor running at approximately $3.0$ GHz. Processing time can be drastically reduced by using parallelisation and better computational resources, hence, we have constructed a practical model for the deterministic case.

Measurement noise can be incorporated in a similar way to how it was done when predicting the radius of an intruder. We assume noisy observations governed by
\begin{equation}
    f^{pq}_i = F_{pq}(x_i, y_i)(1+\eta)
    \label{eq: noisy_model_positions}
\end{equation}
with $\eta\sim\mathcal{N}(0,\sigma^2)$ and we transform them into log-observations with additive noise as was done in \eqref{eq: log_noise}. The dictionary-based model can still be used to infer about the intruder's true position, but this is only done in terms of confidence regions that give us a measure of how likely it is the intruder's exact position was in the region given its predicted position from the noisy observation. Figure~\ref{fig: pred_positions_min_noisy} demonstrates this idea by visualising $95\%$ confidence regions for $F_{15}$ and $F_{24}$ separately. One could take their intersection for a smaller confidence region. However, to more accurately predict the position of an intruder given a noisy observation, we will attempt to construct a multi-layer perceptron (MLP) and a Bayesian inference (BI) model with the hope that they will filter out some of the noise.

\begin{figure}
\centering
\includegraphics[width=0.5\textwidth]{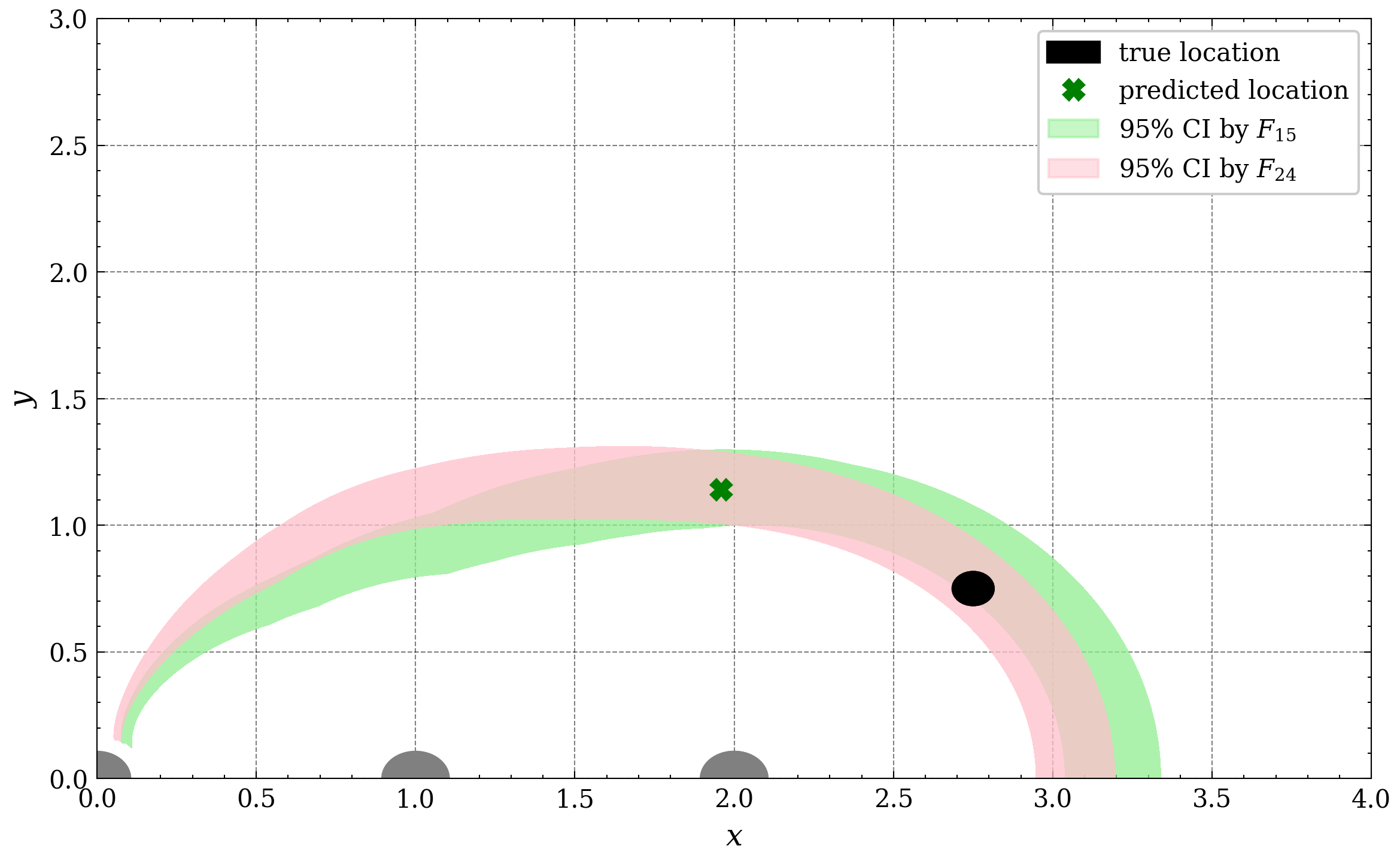}
\caption{Dictionary-based model applied to a noisy observation. Despite not being able to accurately use the deterministic dictionary-based model for prediction when measurement noise is involved, confidence regions can still be determined using the sensing functions $F_{15}$ and $F_{24}$ separately. The $95\%$ confidence intervals are shown in green and pink, respectively. The resonators are shown in gray.}
\label{fig: pred_positions_min_noisy}
\end{figure}

\subsubsection{Multi-layer perceptron}

We will use fully-connected neural networks known as multi-layer perceptrons (MLPs), which are well suited to fitting ill-posed functions \cite{prince2023understanding}. Our MLP is a neural network consisting of five layers. The input layer takes in two-dimensional data, the samples $(f^{15},f^{24})$, and the output layer returns the intruder's coordinates $(x,y)$, which are also two-dimensional. Each of the three hidden layers comprises of $32$ neurons, so the network has $3330$ trainable parameters in total, including bias terms and weights. We apply the ReLU activation function between the layers and we use the Adam optimiser \cite{kingma2015adam} to minimise the mean-squared loss function. To train the network, we take a large dataset of $2000$ points sampled uniformly at random from the domain $[0.1,4]\times[0,3]$. We train the network in batches and apply early-stopping criteria to prevent overfitting. To do this, we split the dataset into three subsets - $70\%$ for training, $20\%$ for validation, and $10\%$ for testing. We then split the training and validation data into batches of size $30$ and monitor both the training loss and validation loss. The idea is that the training loss will only continue to decrease or stagnate while the validation loss will eventually begin increasing again, at which point the model is beginning to overfit and we stop training. In reality we do not know when this will happen, so we let the training loop run slightly longer, but keep track of the optimal parameters that give us the lowest validation loss. The losses from the training process are shown in Figure~\ref{fig: MLP_losses} with the optimal model marked by a horizontal dashed line. The MLP works notably better than the deterministic dictionary model (DM) applied to noisy observations of two different noise levels as demonstrated by Figure~\ref{fig: MLP_DM_preds}, where the MLP predictions are consistently closer to the true values.

\begin{figure}
\centering
\includegraphics[width=0.45\textwidth]{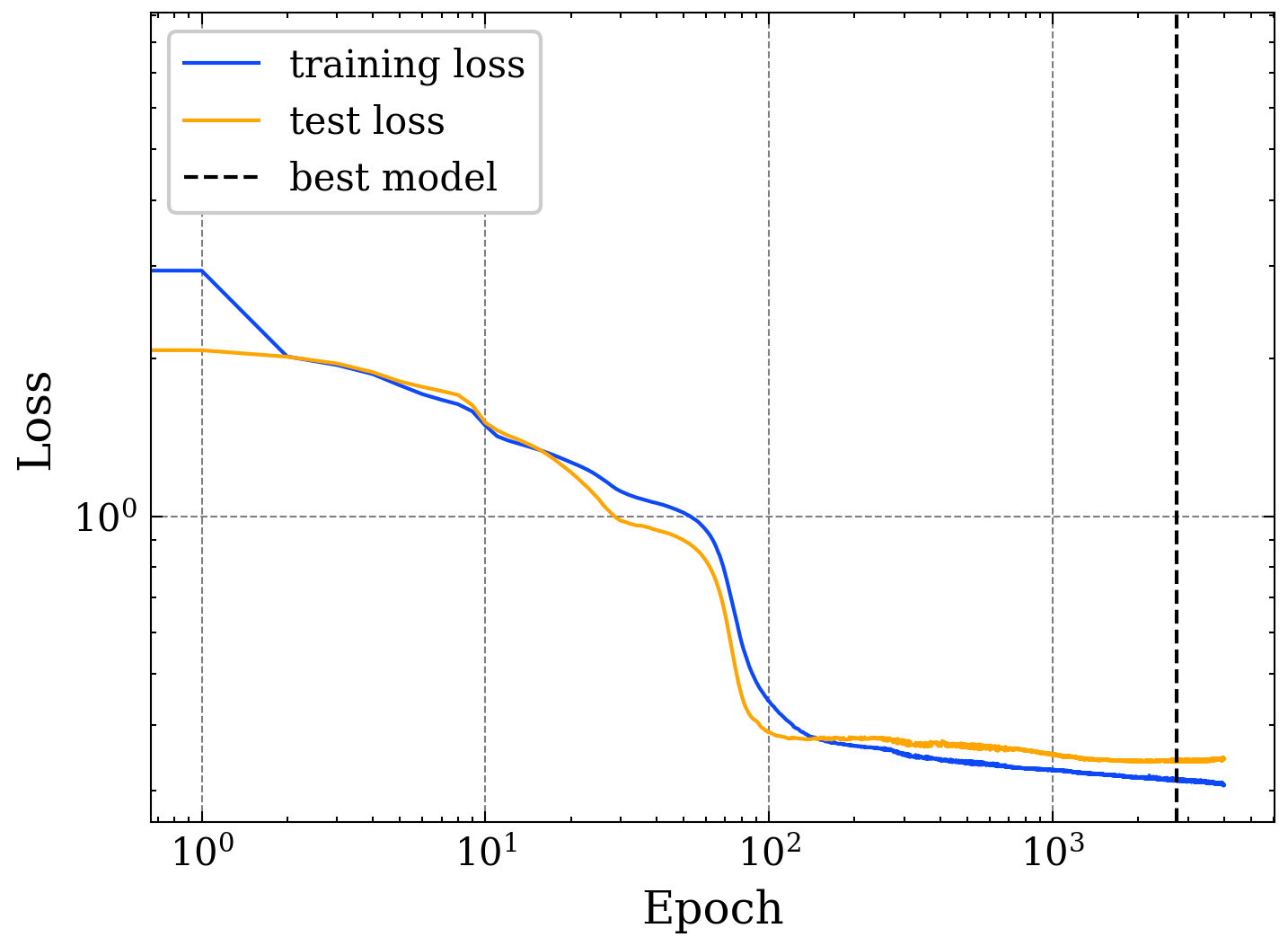}
\caption{Training and testing losses during the tuning of hyperparameters of the multi-layer perceptron (MLP) model. Training loss is with respect to the training dataset, and the test loss is with respect to the validation dataset. Each epoch consists of a batch size of $30$ data samples for both training and testing.}
\label{fig: MLP_losses}
\end{figure}

\begin{figure}
\centering
\includegraphics[width=0.4\textwidth]{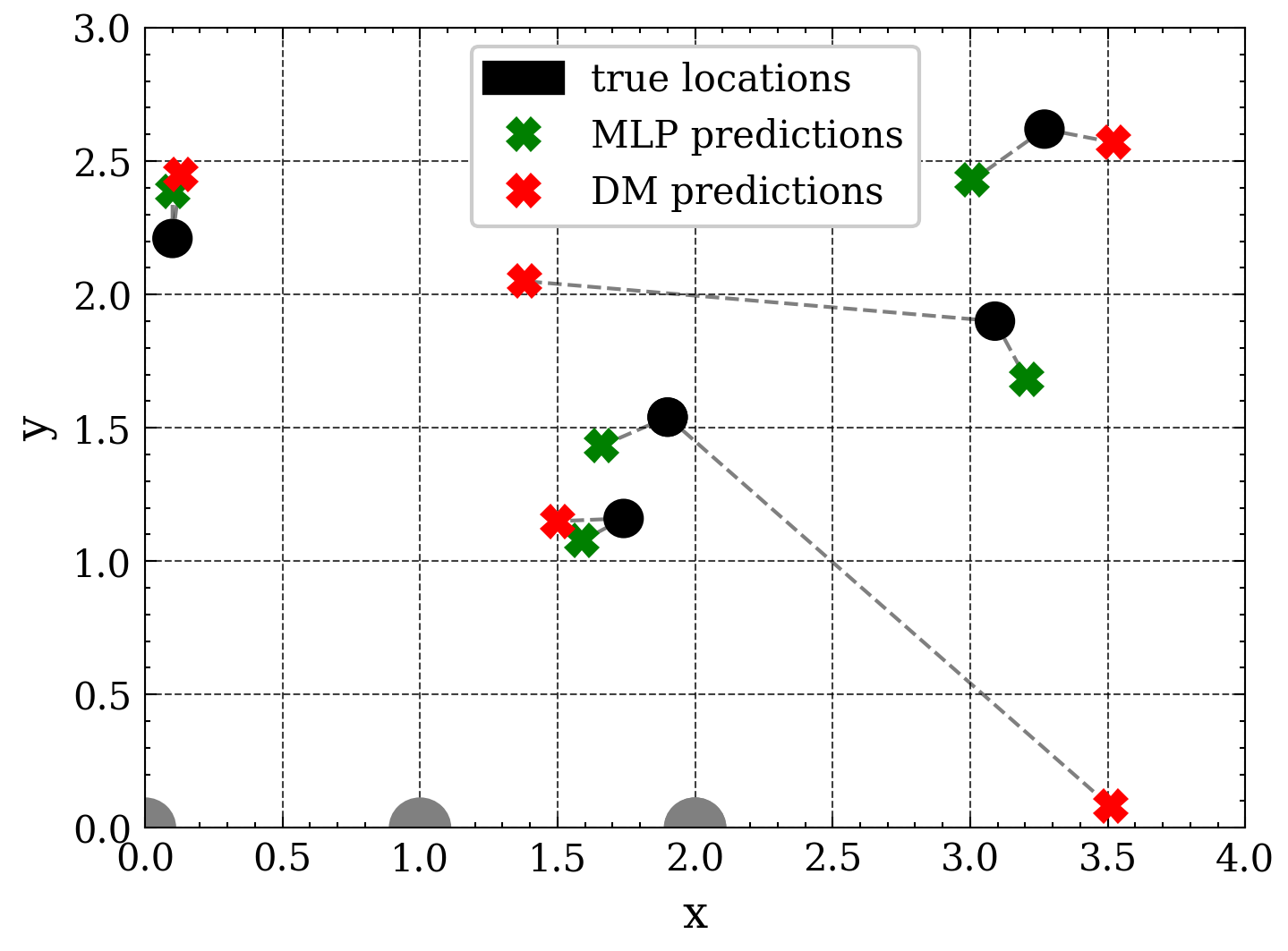}
\includegraphics[width=0.4\textwidth]{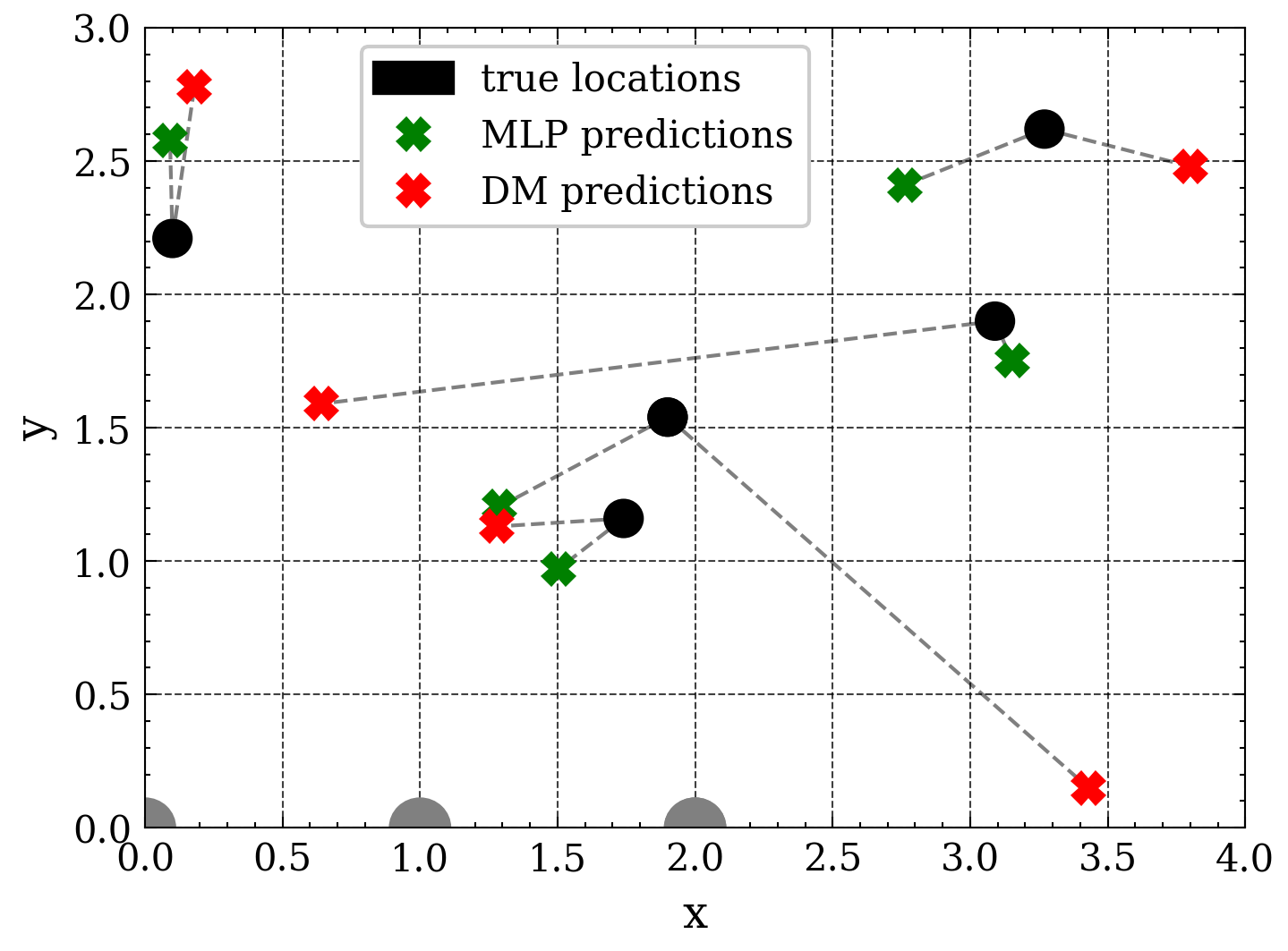}
\caption{Visualisation of several predictions by the multi-layer perceptron (MLP) and the dictionary-based model (DM). The figure on the left shows predictions of samples with measurement noise $\sigma=0.05$, and the figure on the right shows predictions of the same samples but with measurement noise $\sigma=0.1$. The resonators are shown in gray.}
\label{fig: MLP_DM_preds}
\end{figure}

\subsubsection{Bayesian inference}

Finally, we consider the problem from a Bayesian point of view. The idea is to treat the inputs and outputs as random variables, and the goal is to obtain a posterior distribution over the inputs given likelihood and prior distributions. Given that our noisy model \eqref{eq: noisy_model_positions} has multiplicative noise, we can convert it to additive noise by taking logarithms. The log-model now has the form $\tilde{f}^{pq}\approx\widetilde{F}_{pq}+\eta$, where $\tilde{f}^{pq}=\log{f^{pq}}$ and $\widetilde{F}_{pq}=\log{F_{pq}}$. Using Bayes' rule, the posterior can be written as
\begin{equation}
    p\left(x,y\big|\tilde{f}^{15},\tilde{f}^{24}\right) \propto p\left(\tilde{f}^{15},\tilde{f}^{24}\big|x,y\right)\, p(x,y).
\end{equation}
Since $\eta$ is assumed to be normally distributed with zero mean and variance $\sigma^2$, we have an analytic expression for the density function of the likelihood distribution with isotropic covariance matrix $\sigma^2I$,
\begin{equation}
    p(\tilde{f}^{15},\tilde{f}^{24}\big|x,y) = \frac{1}{2\pi\sigma^2}\exp\left(-\frac{1}{2\sigma^2}u^Tu\right),
\end{equation}
where $u$ denotes the vector $(\tilde{f}^{15}-\widetilde{F}_{15}(x,y),\, \tilde{f}^{24}-\widetilde{F}_{24}(x,y))^T$. For the prior distribution, we will consider a uniform distribution over the domain of interest, assuming there is equal chance of finding the intruder anywhere in the domain. However, in some situations one could assume a Gaussian prior which makes the posterior analytically tractable. For example, if there is a heat source in the domain and the intruder we are trying to detect is a virus or some small organism that is attracted to heat, we could use the center of the source and its temperature as the mean and variance of the prior.

Due to the highly non-convex landscape, MAP approximation for the intruder's coordinates is difficult to implement as optimisation algorithms struggle greatly to find the optimal coordinates. However, we can use the affine-invariant ensemble sampler for Markov chain Monte Carlo (MCMC) implemented in the \texttt{emcee} Python software \cite{foremanmackey2013emcee, goodman2010ensemble}. A key advantage of this algorithm is that it can yield short autocorrelation times even for skewed posteriors such as ours.

Let us test this on the point $(2, 1)$ with measurement noise $\sigma = 0.1$. For the first experiment, we assume a uniform prior over the domain of interest. We use the stretch-move sampler with scale parameter $a = 6$, chosen so that the mean acceptance rate lies within the recommended range of $0.2–0.5$ \cite{foremanmackey2013emcee}. For this run it was $0.26$. We simulate 16 walkers and, after discarding a short initial burn-in, run each for 10000 steps. The approximate independence of the samples is quantified by the integrated autocorrelation time $\tau$, which estimates the number of steps required between independent samples \cite{goodman2010ensemble}. We find $\tau\approx 150$, so the chains span over $60$ autocorrelation times, comfortably exceeding the recommended run length of several autocorrelation times \cite{foremanmackey2013emcee}. We then discard a further $5\tau$ steps and thin the chains by keeping every $\tau$-th sample, retaining $960$ weakly correlated samples. To make our prediction we consider two approaches: the \textit{mean prediction}, given by the sample mean over the thinned samples, and the \textit{probability prediction}, given by the mean of the ten samples that attained the highest posterior log-probability. These predictions are plotted alongside the true location and the spread of the samples in Figure \ref{fig: BI_hist}. The probability prediction $(1.89, 0.99)$ lies at a Euclidean distance of $0.11$ from the true location, whereas the mean prediction $(2.06, 0.82)$ lies at a distance of $0.19$. This can be understood from the histograms in Figure \ref{fig: BI_hist}, where the peak of the distribution is near the true $y = 1$ coordinate, but the distribution in $x$ is skewed, which biases the mean. When only the ten samples of highest posterior probability are averaged, the method more accurately predicts the true location.

We repeat the experiment with a Gaussian prior of unit covariance centred at the true location $(2, 1)$, with all sampler settings unchanged. The mean acceptance rate is $0.29$ and the autocorrelation time decreases to $\tau\approx 86$, reflecting the fact that the prior concentrates the posterior mass and makes it easier to explore. Thinning as before retains $1744$ weakly correlated samples. The results are shown in Figure \ref{fig: BI_hist}. The additional prior information sharpens the posterior, with the standard deviations decreasing from $(0.63, 0.25)$ previously with the uniform prior to $(0.50, 0.21)$. This improves both predictions: the mean prediction $(2.13, 0.88)$ lies at a distance of $0.18$ from the true location, while the probability prediction $(1.99, 1.01)$ is within $0.02$ of it. In both experiments the probability prediction outperforms the mean prediction, which suggests that for skewed posteriors such as ours, the high-probability samples are a more reliable summary than the posterior mean. This is likely due to the fact that skewed distributions shift the mean away from the mode, and since the posterior here is unimodal, the probability prediction remains near the posterior mode.

 \begin{figure}
     \centering
     \includegraphics[width=0.8\linewidth, trim={4 0 0 0}, clip]{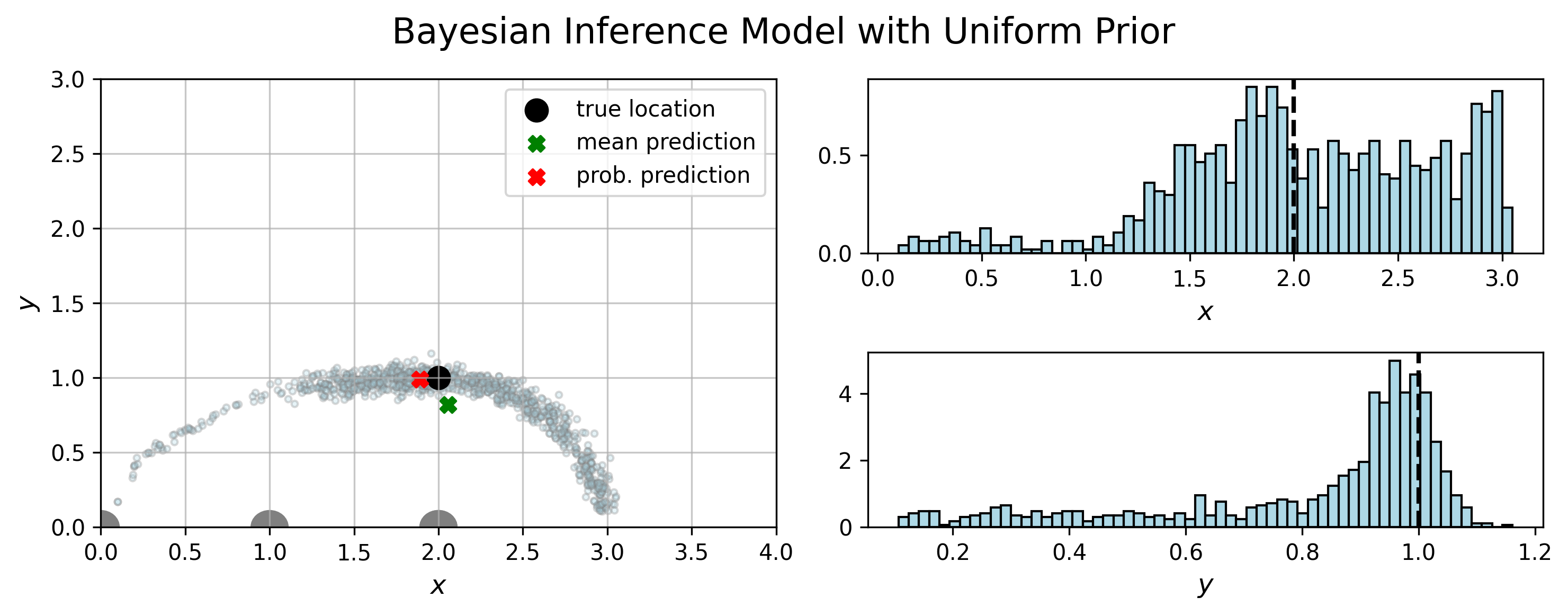}
     \includegraphics[width=0.8\linewidth]{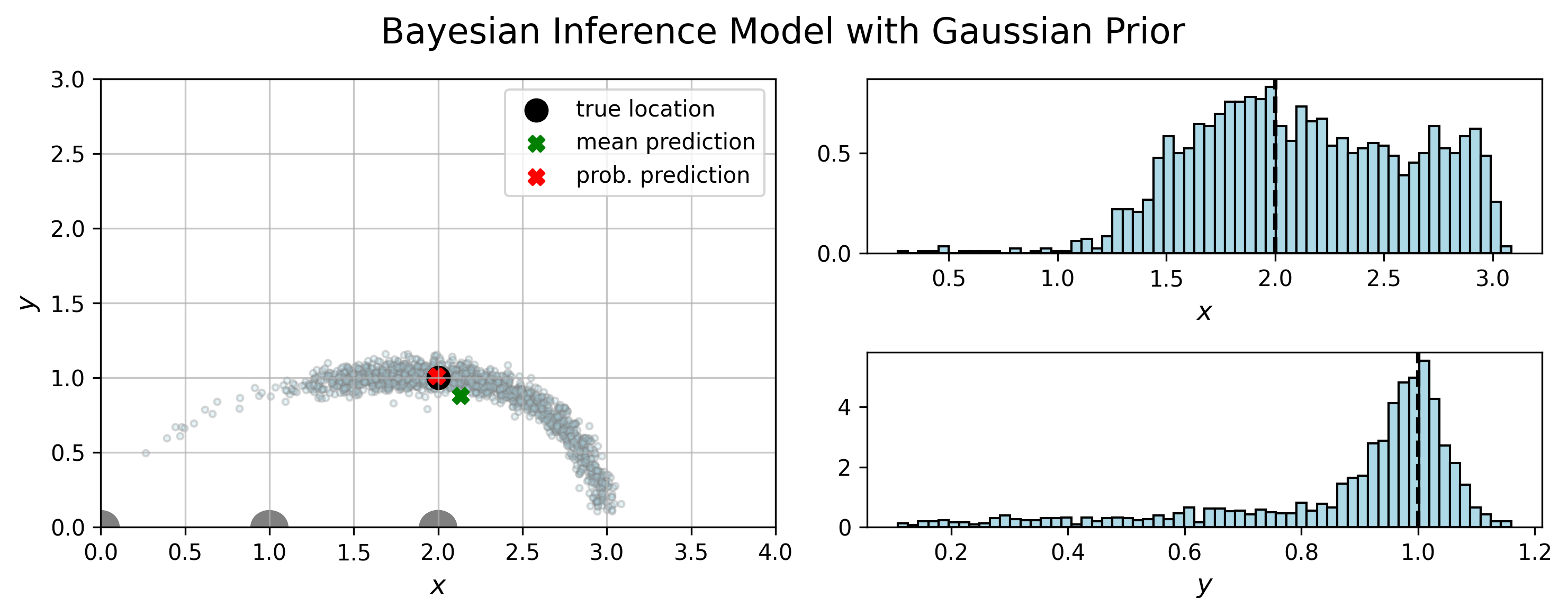}
     \caption{Probability distributions for the Bayesian inference model with uniform and Gaussian priors. On the left, the joint posterior probability distribution for $(x,y)$ is shown, and on the right are histograms representing the marginal distributions. The dashed black horizontal lines in the histograms correspond to the true coordinates. Both predictions are computed from the thinned samples: the mean prediction averages over the full set, whereas the probability prediction averages only over the ten with the highest posterior probability. The resonators are shown in gray (at the bottom of the left-hand plots).}
     \label{fig: BI_hist}
 \end{figure}

\subsection{Predicting radius and position jointly}\label{sec: joint prediction}
In Sections \ref{sec: predicting radius} and \ref{sec: predicting position} we treated the recovery of the intruder's radius $r_\Omega$ and position $(x,y)$ as two separate problems. It is natural to ask whether the two can instead be recovered jointly. As we have established, each symmetric pair contributes a single observation $F_{pq}(r_\Omega,x,y)$, so counting degrees of freedom, we expect that at least three sensing pairs are necessary to recover the three unknowns. Whether three are also sufficient for a unique recovery is what we test below. Our aim here is to establish identifiability in principle, separately from the robustness to noise already examined for the two-dimensional problem, so we work in the noise-free setting and use the dictionary-based model of Section \ref{sec: dictionary model}.

\begin{figure}
    \centering
    \begin{tikzpicture}
    \node at (0,0) (plots) {\includegraphics[width=0.8\textwidth]{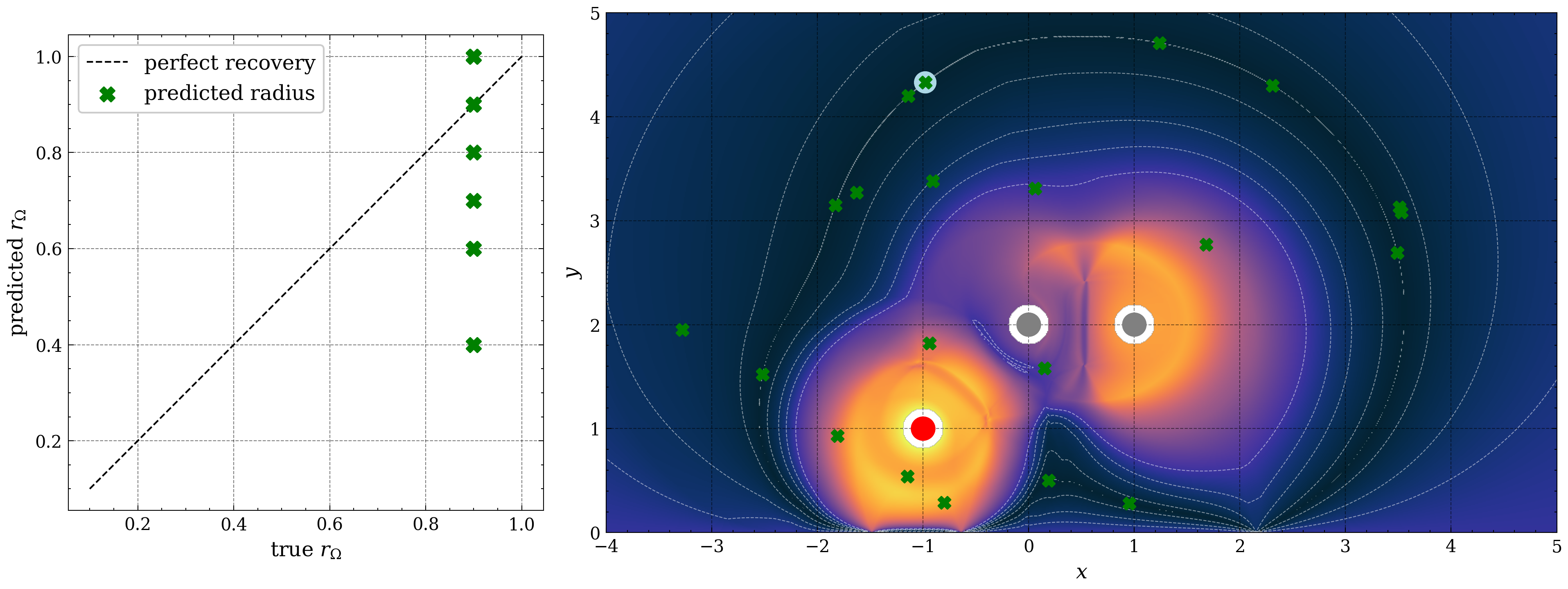}};
    \node[white,scale=0.75,anchor=west] at (4,2) {(a) One pair};
    \begin{scope}[yshift=-4.7cm]
    \node at (0,0) (plots) {\includegraphics[width=0.8\textwidth]{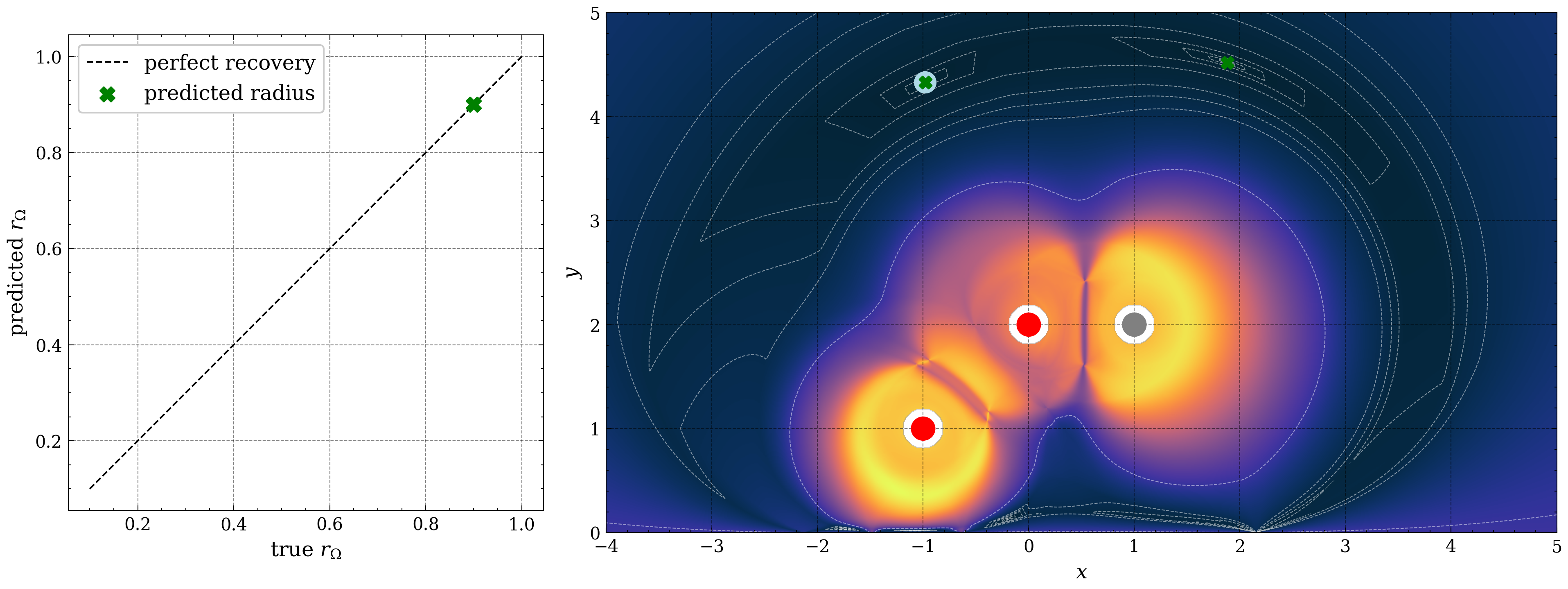}};
    \node[white,scale=0.75,anchor=west] at (4,2) {(b) Two pairs};
    \end{scope}
    \begin{scope}[yshift=-9.4cm]
    \node at (0,0) (plots) {\includegraphics[width=0.8\textwidth]{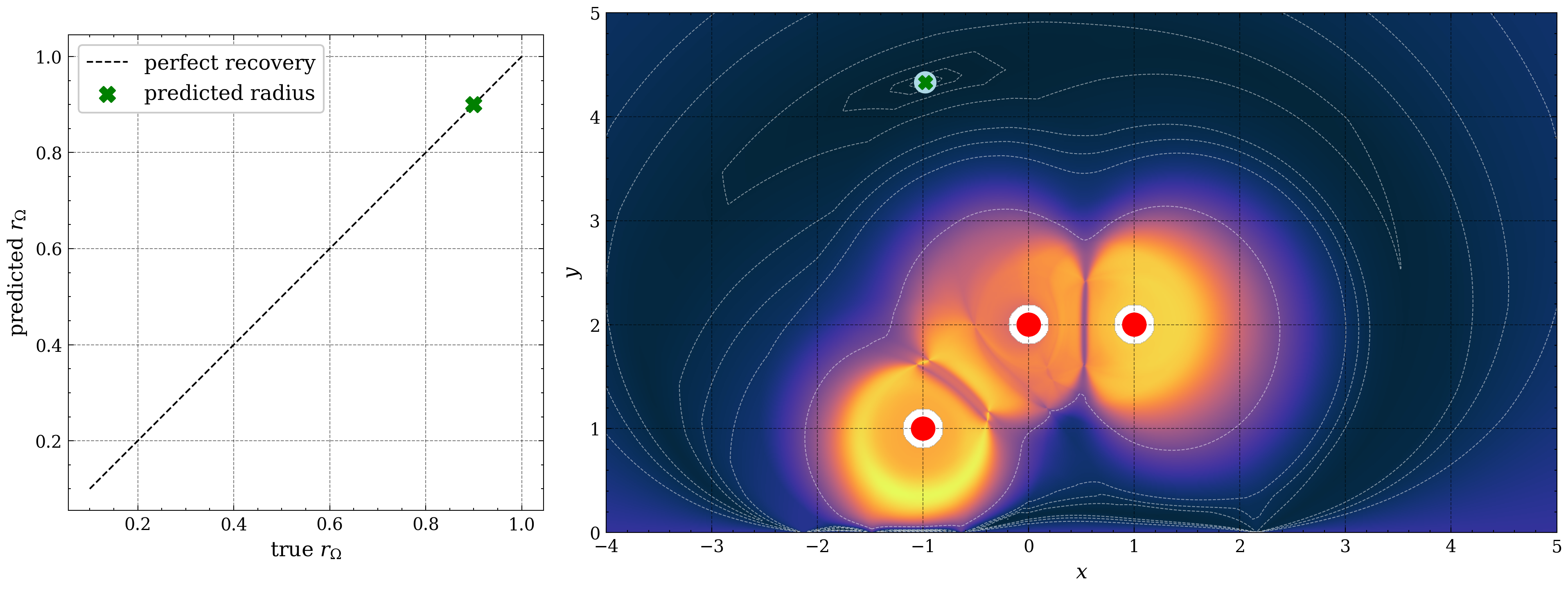}};
    \node[white,scale=0.75,anchor=west] at (4,2) {(c) Three pairs};
    \end{scope}
    \end{tikzpicture}
    \caption{Joint prediction of an intruder's radius and position using a dictionary-based model on the horseshoe arrangement from Figure \ref{fig: horseshoe_pos_many}. Each panel shows the loss surface $F_\text{loss}$ over the position domain, using (a) one, (b) two and (c) three symmetric pairs, respectively. The predictions are given by green crosses, while the true location is marked in white. The sensing pair(s) are shown in red, and the remaining resonators in gray. The intruder has radius $r_\Omega = 0.9$ and position $(x, y) = (-0.98, 4.33)$. Left-hand panels show predicted vs. true radius.}
    \label{fig: joint prediction}
\end{figure}
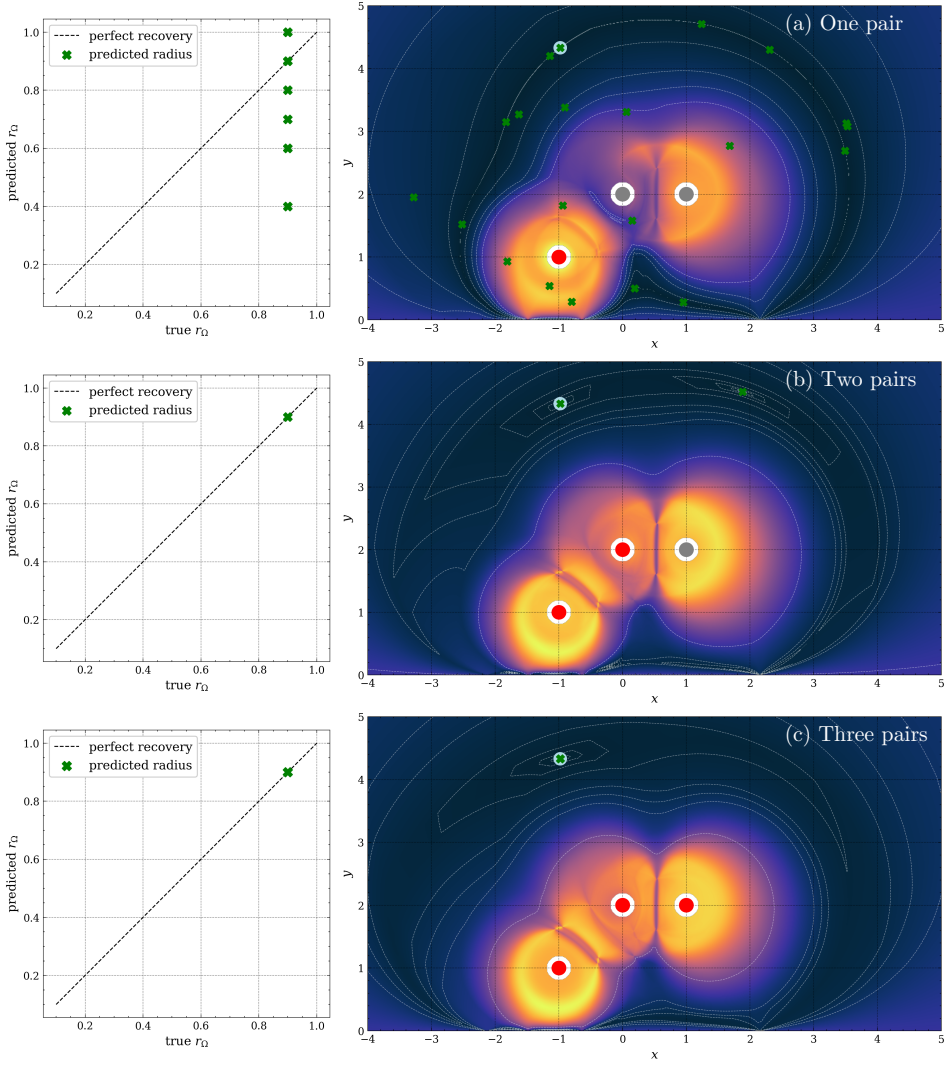

We will test this using the horseshoe arrangement from Figure \ref{fig: horseshoe_pos_many}, whose single mirror symmetry about $y=0$ restricts our position-sensing domain to the upper half plane. We extend the dictionary-based model by discretising the position domain $[-4,5]\times[0,5]$ together with the radius interval $r_\Omega\in[0.1,1.0]$, with steps $\Delta_{x,y}=0.01\times0.01$ and $\Delta_{r_\Omega}=0.1$. We precompute $F_{pq}(r_\Omega,x,y)$ on this grid for each of the three symmetric pairs of sites. Then, given observations $f_{pq}$ from the pairs, we minimise the combined loss
\begin{equation}
    F_\text{loss}(r_\Omega,x,y) = \sum_{p,q}\left| F_{pq}(r_\Omega,x,y)-f_{pq} \right|,
\end{equation}
summed over the symmetric pairs. This generalises (\ref{eq:Fcombined}) to three dimensions, and the grid point attaining the minimum of the loss landscape is our prediction.

Figure~\ref{fig: joint prediction} illustrates this for an intruder at $(r_\Omega,x,y)=(0.9, -0.98, 4.33)$. Using a single symmetric pair, the predicted point is far from unique. As a result, in Figure~\ref{fig: joint prediction}(a) we see that candidates consistent with the observation are scattered across almost the entire domain and the predicted radius ranges widely from the true value. This is the expected behaviour for a map from $\R^3$ to $\R$ and mirrors the non-uniqueness already reported for position recovery from a single pair in Section \ref{sec: dictionary model}. Using two symmetric pairs narrows down the candidates but still underdetermines the system, as the loss surface retains two distinct valleys for the minimum (hence, we see two different predictions for the position in Figure~\ref{fig: joint prediction}(b)). However, when all three symmetries are combined, the surface collapses to a single valley, giving a unique prediction that exactly recovers the true point.

The example shown here is illustrative rather than exhaustive; a fuller treatment assessing accuracy across the whole domain, and examining how predictions behave under measurement noise is a natural extension. The present results are nonetheless enough to show that the joint inverse problem is not inherently ill-posed, and that the dictionary-based approach extends naturally to the joint setting.

%% file: 5_Conclusions.tex
\section{Concluding remarks}
\label{sec:conclusions}

In this work, we have shown that an array of three-dimensional scatterers can be designed to have latent (hidden) symmetry and have methods to use this to build a sensor, capable of localizing a target. These methods can be as simple as dictionary matching, however more sophisticated approaches, such as Bayesian inference or an artificial neural network (multi-layer perceptron), are more effective when noise is added to the data. To our knowledge, this is the first time latent symmetries have been exploited successfully for sensing problems. It is also the first time latent symmetries have been observed in a three-dimensional open system that cannot be approximated by a sparse graph.

There are several possible extensions to this concept.
A straightforward one is to use geometrically asymmetric systems with more than one pair of latently symmetric scatterers.
This lifts the degeneracy of different positions, thus further enhancing the sensing capabilities. This is hard to achieve in open systems (we failed to find any suitable configurations) but relatively straightforward in networks. 
Secondly, we could replace latent reflection symmetry by more general concepts where the isospectral reduction commutes not with a permutation matrix $M$, but with a generic orthogonal matrix $Q$.
For such a more generalized latent symmetry of two scatterers $u,v$, all eigenmodes have \emph{scaled} parity on them.
A special case of this is the recently introduced concept of fractional cospectrality \cite{Chan2020AQFundamentalsFractionalRevivalGraphs}. Finally, we could develop more sophisticated algorithms for recovering the position and radius from the measured data. While the methods applied here were inexpensive to train and performed well on our problem (including in the presence of simulated noise), they are unlikely to generalise well beyond the training set (as is common of supervised models) and may also struggle in more complex physical environments (e.g. with inhomogeneous background media). One promising direction would be to explore bio-inspired sensing algorithms, as evolution has a range of different solutions to these challenges, for instance using networks with various nonlinearities and randomised connections \cite{ammari2020mimicking, ammari2012statistical, bruhin2022bioinspired, christiansen2025morphogenesis, hayes2002distributed, li2006moth, yang2010new, zhou2024bioinspired}.